\documentclass{article}

\usepackage[final]{nips_2016}

\usepackage{booktabs}
\usepackage{amsmath}
\usepackage{mathrsfs}
\usepackage{amsfonts}
\usepackage{amsthm}
\usepackage{booktabs}
\usepackage{stmaryrd}
\usepackage{algorithm}
\usepackage[noend]{algorithmic}
\usepackage{url}
\usepackage{relsize}
\usepackage{pifont}
\usepackage{multirow}
\usepackage{color}
\usepackage{mathtools}

\usepackage{graphicx}

\newcommand{\beq}{\begin{equation}}
\newcommand{\eeq}{\end{equation}}

\newcommand{\beqa}{\begin{eqnarray}}
\newcommand{\eeqa}{\end{eqnarray}}

\newcommand{\beqan}{\begin{eqnarray*}}
\newcommand{\eeqan}{\end{eqnarray*}}

\newcommand{\eqnref}[1]{(\ref{eq:#1})}

\renewcommand{\P}{\mathbb{P}}
\renewcommand{\Pr}{\mathbb{P}}

\newcommand{\E}{\mathbb{E}}

\newcommand{\Var}{\mathbb{V}}
\newcommand{\indic}[1]{\mathbb{I}\{#1\}}
\newcommand{\1}{\mathbb{I} }

\let\R\undefined 
\newcommand{\Real}{\mathbb{R}}

\newcommand{\bN}{\mathbb{N}}
\newcommand{\cbar}{\,|\,}

\newtheorem{predefinition}{Definition}

\newtheorem{theorem}{Theorem}

\newtheorem{preproposition}{Proposition}

\newtheorem{lemma}{Lemma}

\newtheorem{remark}{Remark}

\renewcommand{\hat}{\widehat}
\renewcommand{\phi}{\varphi}
\renewcommand{\epsilon}{\varepsilon}

\newcommand{\A}{\mathcal{A}}

\newcommand{\F}{\mathcal{F}}

\newcommand{\X}{\mathcal{X}}

\newcommand{\T}{\mathcal{T}}

\newcommand{\R}{\mathcal{R}}

\newcommand{\Rmax}{R_{\textsc{max}}}

\DeclareMathOperator*{\expect}{{\mathbb{E}}}

\def \Pcmu {P^{c\mu_k}}
\def \Tpik {\T^{\pi_k}}

\newcommand{\highlight}[1]{\noindent \textbf{#1}}
\newcommand{\noappendix}[1]{#1}

\title{Safe and efficient off-policy reinforcement learning}

\author{
R\'emi Munos\\
munos@google.com \\
Google DeepMind
\And
Thomas Stepleton \\
stepleton@google.com \\
Google DeepMind
\AND
Anna Harutyunyan \\
anna.harutyunyan@vub.ac.be \\
Vrije Universiteit Brussel
\And
Marc G. Bellemare \\
bellemare@google.com \\
Google DeepMind
}

\begin{document}

\maketitle

\begin{abstract}
In this work, we take a fresh look at some old and new algorithms for off-policy, return-based
reinforcement learning. Expressing these in a common form, we derive a novel algorithm,
Retrace($\lambda$), with three desired properties: (1) it has {\em low variance}; (2) 
it {\em safely} uses samples collected from any behaviour policy, whatever its degree of ``off-policyness''; and 
(3) it is {\em efficient} as it makes the best use of samples collected from near on-policy behaviour policies. 
We analyze the contractive nature of the related operator under both off-policy policy evaluation and control settings and derive online sample-based algorithms. We believe this is the first return-based
off-policy control algorithm converging a.s.~to $Q^*$ without the GLIE assumption (Greedy in the Limit with Infinite Exploration). As a corollary, we prove the convergence of Watkins' Q($\lambda$), which was  an open problem since 1989. 
We illustrate the benefits of Retrace($\lambda$) on a standard suite of Atari 2600 games.
\end{abstract}

One fundamental trade-off in reinforcement learning lies in the definition of the update
target: should one estimate Monte Carlo returns or bootstrap from an existing Q-function?
Return-based methods (where {\em return} refers to the sum of discounted rewards $\sum_t \gamma^t r_t$) offer some advantages over value bootstrap methods:
they are better behaved when combined with function approximation, and quickly propagate the fruits
of exploration \citep{sutton1996generalization}.
On the other hand, value bootstrap methods are more readily applied to off-policy data, a common use
case.
In this paper we show that {\em  learning from returns need not be at cross-purposes with off-policy learning. }

We start from the recent work of \citet{harutyunyan16qlambda}, who show that
naive off-policy policy evaluation, without correcting for the ``off-policyness'' of a trajectory,
still converges to the desired $Q^{\pi}$ value function provided the behavior $\mu$ and target $\pi$ policies are not too
far apart (the maximum allowed distance depends on the $\lambda$ parameter). Their $Q^\pi(\lambda)$ algorithm learns from trajectories generated by $\mu$ simply by summing discounted off-policy corrected rewards at each time step. Unfortunately, 
the assumption that $\mu$ and $\pi$ are close is restrictive, as well as difficult to uphold in the control case, where the target policy is greedy with respect to the current Q-function. In that sense this algorithm is not {\em safe}: it does not handle the case of arbitrary ``off-policyness''.

Alternatively, the Tree-backup (TB($\lambda$))  algorithm \citep{precup2000eligibility} tolerates
arbitrary target/behavior discrepancies by scaling information (here called {\em traces}) from future temporal differences by the product of target policy probabilities. TB($\lambda$)
is not {\em efficient} in the ``near on-policy'' case (similar $\mu$ and $\pi$), though, as traces may be cut prematurely, blocking learning from full returns.

In this work, we express several off-policy, return-based algorithms in a common form. From
this we derive an improved algorithm, Retrace($\lambda$), which is both {\em safe} and {\em efficient}, enjoying convergence guarantees for off-policy policy evaluation and -- more importantly -- for the control setting. 

Retrace($\lambda$) can learn from full returns retrieved from past policy data, as
in the context of experience replay \citep{lin93scaling}, which has returned to favour with
advances in deep reinforcement learning \citep{mnih15human,schaul16prioritized}.
Off-policy learning is also desirable for exploration, since it allows the agent to deviate from the target policy currently under evaluation.

To the best of our knowledge, this is the first online return-based off-policy control algorithm which does not require the GLIE (Greedy in the Limit with Infinite Exploration) assumption \citep{singh2000convergence}. In addition, we provide as a corollary the first proof of convergence of Watkins' Q$(\lambda)$ 
\citep[see, e.g.,][]{Watkins:1989,sutton-barto98}.

Finally, we illustrate the significance of Retrace($\lambda$) in a deep learning setting by applying it to the suite of Atari 2600
games provided by the Arcade Learning Environment \citep{bellemare13arcade}.

\section{Notation}

We consider an agent interacting with a Markov Decision Process $(\X, \A, \gamma, P, r)$. $\X$ is a finite state space, $\A$ the  action space, $\gamma \in [0,1)$ the discount factor, $P$ the
transition function mapping state-action pairs $(x,a) \in \X \times \A$ to distributions over $\X$, and $r : \X \times \A \to [-\Rmax,\Rmax]$ is the reward function. For notational simplicity we will consider a finite action space, but the case of infinite -- possibly continuous -- action space can be handled by the Retrace($\lambda$) algorithm as well.  A policy $\pi$ is a mapping
from $\X$ to a distribution over $\A$. A Q-function $Q$ maps each state-action pair $(x,a)$ to
a value in $\Real$; in particular, the reward $r$ is a Q-function. For a policy $\pi$ we define 
the operator $P^\pi$:
\begin{equation*}
(P^\pi Q)(x,a) := \sum_{x' \in \X} \sum_{a' \in \A} P(x' \cbar x, a) \pi(a' \cbar x') Q(x', a') .
\end{equation*}
The value function for a policy $\pi$, $Q^\pi$, describes the expected discounted sum of rewards
associated with following $\pi$ from a given state-action pair. Using operator notation, we write
this as
\begin{equation}
  Q^\pi := \sum_{t \ge 0} \gamma^t (P^\pi)^t r. \label{eq:Qpi_def}
\end{equation}
The \emph{Bellman operator} $\T^\pi$ for a policy $\pi$ is defined as $\T^\pi Q  := r + \gamma P^\pi Q$ and its fixed point is $Q^\pi$, i.e. $\T^\pi Q^\pi  = Q^\pi = (I - \gamma P^\pi)^{-1} r$. 
The \emph{Bellman optimality operator} introduces a maximization over the set of policies:
\begin{equation}
\T Q := r + \gamma \max_{\pi} P^\pi Q . \label{eq:T*}
\end{equation}
Its fixed point is $Q^*$, the unique \emph{optimal value function} \citep{Puterman:1994}. It is this quantity that we will seek to obtain when we talk about the ``control setting''.

\paragraph{Return-based Operators:}

The $\lambda$-return extension \citep{sutton1988learning} of the Bellman operators considers exponentially weighted sums of $n$-steps returns:
\begin{equation*}
\T_\lambda^\pi Q := (1 - \lambda) \sum_{n \ge 0} \lambda^n \left [ (\T^\pi)^n Q \right ] = Q + (I - \lambda \gamma P^\pi)^{-1} (\T^\pi Q - Q) ,
\end{equation*}
where $\T^\pi Q - Q$ is the \emph{Bellman residual} of $Q$ for policy $\pi$. Examination of the above
shows that $Q^\pi$ is also the fixed point of $\T_\lambda^\pi$. At one extreme ($\lambda = 0$) we have the Bellman
operator $\T_{\lambda=0}^\pi Q=\T^{\pi}Q$, while at the other ($\lambda = 1$) we have the policy evaluation operator $\T_{\lambda=1}^\pi Q=Q^{\pi}$ which can be estimated using Monte Carlo methods~\citep{sutton-barto98}. Intermediate values of $\lambda$ trade off estimation bias with sample variance \citep{kearns2000bias}.

We seek to evaluate a 
\emph{target policy} $\pi$ using trajectories drawn from a \emph{behaviour policy} $\mu$. 
If $\pi = \mu$, we are \emph{on-policy}; otherwise, we are \emph{off-policy}.
We will consider trajectories of the form:
$$x_0=x, a_0=a, r_0, x_1, a_1, r_1, x_2, a_2, r_2, \dots$$
with $a_t \sim \mu(\cdot|x_t)$, $r_t=r(x_t,a_t)$ and $x_{t+1}\sim P(\cdot|x_t, a_t)$. We denote 
by $\F_t$ this sequence up to time $t$, and write $\E_\mu$ the expectation with
respect to both $\mu$ and the MDP transition probabilities. Throughout, we write
$\|\cdot\|$ for supremum norm.

\section{Off-Policy Algorithms}\label{sec:off_policy_algorithms}

We are interested in two related off-policy learning problems. In the \emph{policy evaluation}
setting, we are given a fixed policy $\pi$ whose value $Q^\pi$ we wish to estimate from sample
trajectories drawn from a behaviour policy $\mu$. In the \emph{control} setting, we consider a
sequence of policies that depend on our own sequence of Q-functions (such as $\epsilon$-greedy policies), 
and seek to approximate $Q^*$. 

The general operator that we consider for comparing several return-based off-policy algorithms is:
\begin{equation}\label{eq:general_operator}
\R Q(x,a) := Q(x,a) +  \E_\mu \Big[ \sum_{t \ge 0} \gamma^t\Big (\prod_{s=1}^t c_s \Big) \big( r_t + \gamma \E_{\pi} Q(x_{t+1},\cdot) - Q(x_t, a_t) \big) \Big], 
\end{equation}
for some non-negative coefficients $(c_s)$, where we write $\E_{\pi} Q(x,\cdot) := \sum_{a} \pi(a|x) Q(x,a)$ and define $(\prod_{s=1}^t c_s)=1$ when $t=0$. By extension of the idea of eligibility traces \citep{sutton-barto98}, we informally call the coefficients $(c_s)$ the \emph{traces} of the operator. 
\vspace{-0.2cm}

\paragraph{Importance sampling (IS): $c_s=\frac{\pi(a_s|x_s)}{\mu(a_s|x_s)}$.} Importance sampling is the simplest way to correct for the discrepancy between $\mu$ and $\pi$ when learning from off-policy returns \citep{precup2000eligibility,precup01offpolicy,geist2014off}. The off-policy correction uses the product of the likelihood ratios between $\pi$ and $\mu$.
Notice that $\R Q$ defined in \eqref{eq:general_operator}  with this choice of $(c_s)$ yields $Q^{\pi}$ for any $Q$. For $Q=0$ we recover the basic IS estimate $\sum_{t \ge 0} \gamma^t \big (\prod_{s=1}^t c_s \big) r_t$, thus \eqref{eq:general_operator} can be seen as a variance reduction technique (with a baseline $Q$).
It is well known that IS estimates can suffer from large -- even possibly infinite -- variance (mainly due to the variance of the product $\frac{\pi(a_1|x_1)}{\mu(a_1|x_1)}\cdots \frac{\pi(a_t|x_t)}{\mu(a_t|x_t)}$), which has motivated further variance reduction techniques such as in \citep{mahmood2015off,mahmood2015emphatic,1509.05172}. 
\vspace{-0.2cm}

\paragraph{Off-policy Q$^\pi$($\lambda$) and Q$^*$($\lambda$): $c_s=\lambda$.} A recent alternative proposed by \citet{harutyunyan16qlambda} introduces an off-policy correction based on a $Q$-baseline (instead of correcting the probability of the sample path like in IS). This approach, called Q$^\pi$($\lambda$) and Q$^*$($\lambda$) for policy evaluation and control, respectively, corresponds to the choice $c_s=\lambda$. It offers the advantage of avoiding the blow-up of the variance of the product of ratios encountered with IS. Interestingly, this operator contracts around $Q^\pi$ provided that $\mu$ and $\pi$ are sufficiently close to each other. Defining $\epsilon := \max_x \| \pi(\cdot|x) - \mu(\cdot|x)\|_1$ the level of ``off-policyness'', the authors prove that the operator defined by \eqref{eq:general_operator} with $c_s=\lambda$ is a contraction mapping around $Q^{\pi}$ for $\lambda < \frac{1 - \gamma}{\gamma\epsilon}$, and around $Q^*$ for the worst case of $\lambda < \frac{1 - \gamma}{2\gamma}$. Unfortunately, Q$^\pi$($\lambda$) requires knowledge of $\epsilon$, and the condition for Q$^*$($\lambda$) is very conservative. Neither Q$^\pi$($\lambda$), nor Q$^*$($\lambda$) are safe as they do not guarantee convergence for arbitrary $\pi$ and $\mu$. 
\vspace{-0.2cm}

\paragraph{Tree-backup, TB($\lambda$): $c_s = \lambda \pi(a_s|x_s)$.} The TB($\lambda$) algorithm of \citet{precup2000eligibility}
corrects for the target/behaviour discrepancy by multiplying each term of the sum by the product
of target policy probabilities. The corresponding operator defines a contraction mapping for any policies $\pi$ and $\mu$, which makes it a safe algorithm. However, this algorithm is not efficient in the near on-policy case (where $\mu$ and $\pi$ are similar) as it unnecessarily cuts the traces, preventing it to make use of full returns: indeed we need not discount stochastic on-policy transitions (as shown by \citeauthor{harutyunyan16qlambda}'s results about Q$^\pi$). 

\vspace{-0.2cm}

\paragraph{Retrace($\lambda$): 
$c_s=\lambda \min\Big(1, \frac{\pi(a_s|x_s)}{\mu(a_s|x_s)}\Big).$}
Our contribution is an algorithm -- Retrace$(\lambda)$ -- that takes the best of the three previous algorithms.
Retrace$(\lambda)$ uses an importance sampling ratio truncated at $1$. Compared to IS, it does not suffer from the variance explosion of the product of IS ratios. Now, similarly to $Q^{\pi}(\lambda)$ and unlike TB($\lambda$), it does not cut the traces in the on-policy case, making it possible to benefit from the full returns. In the off-policy case, the traces are safely cut, similarly to TB($\lambda$). In particular, $\min\big(1, \frac{\pi(a_s|x_s)}{\mu(a_s|x_s)}\big)\geq \pi(a_s|x_s)$: Retrace($\lambda$) does not cut the traces as much as TB($\lambda$).
In the subsequent sections, we will show the following:
\begin{itemize}
    \item For any traces $0\leq c_s\leq \pi(a_s|x_s) / \mu(a_s|x_s)$ (thus including the Retrace($\lambda$) operator), the return-based operator \eqref{eq:general_operator} is a $\gamma$-contraction around $Q^\pi$, for {\em arbitrary} policies $\mu$ and $\pi$
    \item In the control case (where $\pi$ is replaced by a sequence of increasingly greedy policies) the online Retrace($\lambda$) algorithm converges a.s.~to $Q^*$, without requiring the GLIE assumption.
    \item As a corollary, Watkins's Q$(\lambda)$ converges a.s.~to $Q^*$.
\end{itemize}

\begin{table}
\begin{center}
\begin{tabular}{@{}lcccc@{}}
\toprule
& Definition & Estimation & Guaranteed & Use full returns \\
& of $c_s$ & variance & convergence$\dagger$ & (near on-policy) \\
\midrule
Importance sampling & $\frac{\pi(a_s|x_s)}{\mu(a_s|x_s)}$ & High & for any $\pi$, $\mu$ & yes \\ 
\midrule
$Q^{\pi}(\lambda)$ & $\lambda$ & Low & for $\pi$ close to $\mu$ & yes \\ 
\midrule
TB($\lambda$) &  $\lambda\pi(a_s|x_s)$ & Low & for any $\pi$, $\mu$ & no \\
\midrule
Retrace($\lambda$) &  $\lambda\min\Big(1, \frac{\pi(a_s|x_s)}{\mu(a_s|x_s)}\Big)$ & Low & for any $\pi$, $\mu$  & yes \\
\toprule
\end{tabular}
\end{center}
\caption{Properties of several algorithms defined in terms of the general operator given in \eqnref{general_operator}. $\dagger$Guaranteed convergence of the expected operator $\R$.\label{table:algorithms_table}}
\end{table}

\section{Analysis of Retrace($\lambda$)}

We will in turn analyze both off-policy policy evaluation and control settings.
We will show that $\R$ is a contraction mapping in both settings (under a mild additional assumption for the control case).

\subsection{Policy Evaluation}

Consider a fixed target policy $\pi$. For ease of exposition we consider a fixed behaviour policy $\mu$, noting that our result extends to the setting of sequences of behaviour policies $(\mu_k : k \in \bN)$. 

Our first result states the $\gamma$-contraction of the operator \eqnref{general_operator} defined by any set of non-negative coefficients $c_s = c_s(a_s, \F_s)$ (in order to emphasize that $c_s$ can be a function of the whole history $\F_s$) under the assumption that $0\leq c_s \leq \frac{\pi(a_s | x_s)}{\mu(a_s | x_s)}$.

\begin{theorem}\label{thm:retrace.contraction}
The operator $\R$ defined by \eqnref{general_operator} has a unique fixed point $Q^\pi$. Furthermore, if for each 
$a_s \in \A$ and each history $\F_s$ we have $c_s = c_s(a_s, \F_s) \in \big [0, \frac{\pi(a_s | x_s)}{\mu(a_s | x_s)}\big ]$, then for any Q-function $Q$
$$\| \R Q- Q^{\pi}  \| \leq \gamma \|Q - Q^{\pi}\|.$$
\end{theorem}

The following lemma will be useful in proving Theorem \ref{thm:retrace.contraction} (proof in the appendix).
\begin{lemma}\label{lem:retrace_delta_lemma}
The difference between $\R Q$ and its fixed point $Q^\pi$ is
\begin{equation*}
\label{eq:RQ-Qpi}
\R Q(x,a)- Q^{\pi}(x,a) =\E_{\mu} \Big[ \ \sum_{t\geq 1} \gamma^{t} \Big( \prod_{i=1}^{t-1} c_i\Big) \Big( \big[ \E_\pi [(Q-Q^{\pi})(x_{t}, \cdot)] - c_{t} (Q-Q^{\pi})(x_{t},a_{t}) \big] \Big)\Big].
\end{equation*}
\end{lemma}
\begin{proof}[Proof (Theorem \ref{thm:retrace.contraction})]
The fact that $Q^\pi$ is the fixed point of the operator $\R$ is obvious from \eqnref{general_operator} since
$\E_{x_{t+1}\sim P(\cdot|x_t, a_t)} \big[ r_t +\gamma \E_{\pi} Q^{\pi}(x_{x+1},\cdot)-Q^{\pi}(x_t, a_t)\big] = (\T^{\pi}Q^{\pi}-Q^{\pi})(x_t, a_t) = 0$, since $Q^{\pi}$ is the fixed point of $\T^{\pi}$. Now, from Lemma~\ref{lem:retrace_delta_lemma}, and defining $\Delta Q := Q-Q^{\pi}$, we have
\begin{align*}
\R Q(x,a)- Q^{\pi}(x,a) &= \sum_{t\geq 1} \gamma^{t} \expect_{\mathclap{\substack{x_{1:t}\\a_{1:t}}}} \hspace{0.5em} \Big[ \Big( \prod_{i=1}^{t-1} c_i\Big) \Big( \big[ \E_\pi \Delta Q(x_{t}, \cdot) - c_{t} \Delta Q(x_{t},a_{t}) \big] \Big)\Big] \\
&= \sum_{t\geq 1} \gamma^{t} \expect_{\mathclap{\substack{x_{1:t}\\a_{1:t-1}}}} \hspace{0.5em} \Big[ \Big( \prod_{i=1}^{t-1} c_i\Big) \Big( \big[ \E_\pi \Delta Q(x_{t}, \cdot) - \E_{a_t} [c_{t}(a_t,\F_t) \Delta Q(x_{t},a_{t}) | \F_t] \big] \Big)\Big] \\
&= \sum_{t\geq 1} \gamma^{t} \expect_{\mathclap{\substack{x_{1:t}\\a_{1:t-1}}}} \hspace{0.5em} \Big[ \Big( \prod_{i=1}^{t-1} c_i\Big)  \sum_b \big( \pi(b|x_t) - \mu(b|x_t) c_t(b,\F_t) \big)\Delta Q(x_{t},b) \Big].
\end{align*}
Now since $\pi(a|x_t) - \mu(a|x_t) c_t(b,\F_t) \geq 0$, we have that $\R Q(x,a)- Q^{\pi}(x,a)= \sum_{y,b} w_{y,b}\Delta Q(y,b)$, i.e.~a linear combination of $\Delta Q(y,b)$ weighted by non-negative coefficients:
$$w_{y,b} :=  \sum_{t\geq 1} \gamma^{t} \expect_{\mathclap{\substack{x_{1:t}\\a_{1:t-1}}}} \hspace{0.5em} \Big[ \Big( \prod_{i=1}^{t-1} c_i\Big)  \big( \pi(b|x_t) - \mu(b|x_t) c_t(b,\F_t) \big) \1\{x_t=y\}\Big].$$
The sum of those coefficients is:
\begin{align*}
\sum_{y,b} w_{y,b} & = \sum_{t\geq 1} \gamma^{t} \expect_{\mathclap{\substack{x_{1:t}\\a_{1:t-1}}}} \hspace{0.5em} \Big[ \Big( \prod_{i=1}^{t-1} c_i\Big)  \sum_b \big( \pi(b|x_t) -   \mu(b|x_t)c_t(b,\F_t) \big) \Big] \\
& = \sum_{t\geq 1} \gamma^{t} \expect_{\mathclap{\substack{x_{1:t}\\a_{1:t-1}}}} \hspace{0.5em} \Big[ \Big( \prod_{i=1}^{t-1} c_i\Big)  \E_{a_t} [ 1- c_t(a_t,\F_t) | \F_t] \Big] 
= \sum_{t\geq 1} \gamma^{t} \expect_{\mathclap{\substack{x_{1:t}\\a_{1:t}}}} \hspace{0.5em} \Big[ \Big( \prod_{i=1}^{t-1} c_i\Big)  (1- c_t) \Big] \\
& = \E_{\mu} \Big[ \sum_{t\geq 1} \gamma^{t} \Big( \prod_{i=1}^{t-1} c_i\Big)  -  \sum_{t\geq 1} \gamma^{t} \Big( \prod_{i=1}^{t} c_i\Big) \Big] = \gamma C - (C-1), 
\end{align*}
where $C:= \E_{\mu} \big[ \sum_{t\geq 0} \gamma^{t} \big( \prod_{i=1}^{t} c_i\big)\big]$. Since $C\geq 1$, we have that $\sum_{y,b} w_{y,b} \leq \gamma.$
Thus $\R Q(x,a)- Q^{\pi}(x,a)$ is a sub-convex combination of $\Delta Q(y,b)$ weighted by non-negative coefficients $w_{y,b}$ which sum to (at most) $\gamma$, thus $\R$ is a $\gamma$-contraction mapping around $Q^{\pi}$.
\end{proof}

\begin{remark}\label{rem:eta}
Notice that the coefficient $C$ in the proof of Theorem \ref{thm:retrace.contraction} depends on
$(x,a)$. If we write $\eta(x,a) := 1 - (1 - \gamma) \E_\mu \left [ \sum_{t \ge 0} \gamma^t (\prod_{s=1}^t c_s) \right ]$, then we have shown that 
\begin{equation*}
| \R Q(x,a) - Q^\pi(x,a) | \le \eta(x,a) \| Q - Q^\pi \|.
\end{equation*}
Thus $\eta(x,a) \in [0, \gamma]$ is a $(x,a)$-specific contraction coefficient, which is $\gamma$
when $c_1 = 0$ (the trace is cut immediately) and can be close to zero when learning from full returns ($c_t \approx 1$ for all $t$).
\end{remark}

\subsection{Control}

In the control setting, the single target policy $\pi$ is replaced by a sequence of policies $(\pi_k)$ which
depend on $(Q_k)$. 
While most prior work has focused on strictly greedy policies, here we consider the larger class of \emph{increasingly greedy} sequences. We now make this notion precise. 
\begin{predefinition}
We say that a sequence of policies $(\pi_k : k \in \bN)$ is \emph{increasingly greedy} w.r.t.~a sequence $(Q_k: {k \in \bN})$ of Q-functions  if the following property holds for all $k$:
$P^{\pi_{k+1}}Q_{k+1} \geq P^{\pi_{k}}Q_{k+1}.$
\end{predefinition}
Intuitively, this means that each $\pi_{k+1}$ is at least as greedy as the previous policy $\pi_k$ for $Q_{k+1}$. Many natural sequences of policies are increasingly greedy, including $\epsilon_k$-greedy policies (with non-increasing $\epsilon_k$) and softmax policies (with non-increasing temperature). See proofs in the appendix. 

We will assume that $c_s=c_s(a_s,\F_s)=c(a_s,x_s)$ is Markovian, in the sense that it depends on $x_s, a_s$ (as well as the policies $\pi$ and $\mu$) only but not on the full past history. This allows us to define the (sub)-probability transition operator 
$$(P^{c\mu} Q)(x,a) := \sum_{x'}\sum_{a'} p(x'|x,a)\mu(a'|x') c(a',x') Q(x',a').$$
Finally, an additional requirement to the convergence in the control case, we assume that $Q_0$ satisfies $\T^{\pi_0}Q_0\geq Q_0$ (this can be achieved by a pessimistic initialization $Q_0= -R_{MAX}/(1-\gamma)$).
\begin{theorem}\label{thm:main-result}
Consider an arbitrary sequence of behaviour policies $(\mu_k)$ (which may depend on $(Q_k)$) and a sequence of target policies $(\pi_k)$ that are increasingly greedy w.r.t.~the sequence $(Q_k)$:
$$Q_{k+1} = \R_k Q_k,$$
where the return operator $\R_k$ is defined by \eqref{eq:general_operator} for $\pi_k$ and $\mu_k$ and a Markovian $c_s=c(a_s,x_s)\in [0, \frac{\pi_k(a_s|x_s)}{\mu_k(a_s|x_s)}]$. Assume the target policies $\pi_k$ are $\epsilon_k$-away from the greedy policies w.r.t.~$Q_k$, in the sense that $\T^{\pi_k}Q_k \geq \T Q_k - \epsilon_k \|Q_k\| e$, where $e$ is the vector with 1-components. Further suppose that $\T^{\pi_0}Q_0\geq Q_0$. Then for any $k\geq 0$,
\begin{equation*}
\| Q_{k+1} - Q^*\| \leq \gamma\| Q_k -  Q^{*}\| + \epsilon_k \|Q_k\|.
\end{equation*}
In consequence, if $\epsilon_k \to 0$ then $Q_k \to Q^*$.
\end{theorem}

\begin{proof}[Sketch of Proof (The full proof is in the appendix)]
Using $\Pcmu$, the Retrace$(\lambda)$ operator rewrites 
\begin{equation*}
\R_k Q = Q + \sum_{t \ge 0} \gamma^t (\Pcmu)^t (\Tpik Q - Q) = Q + (I - \gamma \Pcmu)^{-1} (\Tpik Q - Q).
\end{equation*}
We now lower- and upper-bound the term $Q_{k+1} - Q^*$.

\highlight{Upper bound on $Q_{k+1} - Q^*$.} We prove that $Q_{k+1} - Q^*\leq A_k (Q_k-Q^*)$ with $A_k := \gamma (I-\gamma P^{c\mu_k})^{-1} \big[ P^{\pi_k} -  P^{c\mu_k} \big]$.  Since $c_t \in [0, \frac{\pi(a_t|x_t)}{\mu(a_t|x_t)}]$ we deduce that $A_k$ has non-negative elements, whose sum over each row, is at most $\gamma$.
Thus
\beq
Q_{k+1} - Q^* \leq \gamma \|Q_k - Q^*\| e.\label{eq:upper-bound-Qk-Q*}
\eeq

\highlight{Lower bound on $Q_{k+1} - Q^*$.} Using the fact that $\T^{\pi_k}Q_k \geq \T^{\pi^*}Q_k - \epsilon_k\|Q_k\| e$ we have
\beqa
Q_{k+1} - Q^* &\geq& Q_{k+1} - \T^{\pi_k} Q_k  +\gamma P^{\pi^*}  (Q_k -  Q^{*}) - \gamma \epsilon_k\|Q_k\| e \notag \\
&= &\gamma P^{c\mu_k} (I-\gamma P^{c\mu_k})^{-1} (\T^{\pi_k}Q_k - Q_k)  + \gamma P^{\pi^*}  (Q_k -  Q^{*}) - \epsilon_k\|Q_k\|e.\label{eq:lb-Qk+1-Q*}
\eeqa

\highlight{Lower bound on $\Tpik Q_k - Q_k$.} Since the sequence $(\pi_k)$ is increasingly greedy w.r.t.~$(Q_k)$, we have 
\beqa
\T^{\pi_{k+1}} Q_{k+1} - Q_{k+1} &\geq& \T^{\pi_k} Q_{k+1} - Q_{k+1} = r+(\gamma P^{\pi_k}-I) \R_k Q_k \notag \\
&=& B_k (\T^{\pi_k} Q_k - Q_k),\label{eq:lb.TQk+1-Qk+1}
\eeqa
where $B_k:=\gamma [ P^{\pi_k} -   P^{c\mu_k} ] (I-\gamma P^{c\mu_k})^{-1}$.
Since $P^{\pi_k} - \Pcmu$ and $(I-\gamma P^{c\mu_k})^{-1}$ are non-negative matrices, so is $B_k$. Thus
$\T^{\pi_k} Q_k - Q_k \geq B_{k-1} B_{k-2} \dots B_{0} (\T^{\pi_0} Q_0 - Q_0) \geq 0,
$
since we assumed $T^{\pi_0} Q_0 - Q_0 \ge 0$. Thus, \eqref{eq:lb-Qk+1-Q*} implies that
$$
Q_{k+1} - Q^* \geq \gamma P^{\pi^*}  (Q_k -  Q^{*}) - \epsilon_k\|Q_k\|e.
$$
Combining the above with \eqnref{upper-bound-Qk-Q*} we deduce $\| Q_{k+1} - Q^*\| \leq \gamma\| Q_k -  Q^{*}\| +  \epsilon_k \| Q_k \| $. When $\epsilon_k\rightarrow 0$, we further deduce that $Q_k$ are bounded, thus $Q_k \to Q^*$.
\end{proof}

\subsection{Online algorithms}

So far we have analyzed the contraction properties of the expected $\R$ operators. We now describe online algorithms which can learn from sample trajectories. We analyze the algorithms in the {\em every visit} form~\citep{sutton-barto98}, which is the more practical generalization of the first-visit form. In this section, we will only consider the Retrace($\lambda$) algorithm defined with the coefficient $c=\lambda\min(1,\pi/\mu)$. For that $c$, let us rewrite the operator $P^{c\mu}$ as $\lambda P^{\pi\wedge\mu}$, where $P^{\pi\wedge\mu}Q(x,a):=\sum_y \sum_b \min(\pi(b|y),\mu(b|y)) Q(y,b)$, and write the Retrace operator $\R Q = Q+(I-\lambda\gamma P^{\pi\wedge\mu})^{-1}(\T^\pi Q-Q)$. We focus on the control case, noting that a similar (and simpler) result can be derived for policy evaluation.

\begin{theorem}\label{thm:online}
Consider a sequence of sample trajectories, with the $k^{th}$ trajectory $x_0, a_0, r_0, x_1, a_1, r_1, \dots$
generated by following $\mu_k$: $a_t\sim \mu_k(\cdot|x_t)$. For each $(x,a)$ along this trajectory, with $s$ being the time of first occurrence of $(x, a)$, update
\begin{align}
 Q_{k+1}(x, a) & \gets  Q_k(x, a) + \alpha_k\sum_{t \ge s}
\delta^{\pi_k}_t \sum_{j = s}^t \gamma^{t - j} \Big (
  \prod_{i=j+1}^t c_i \Big ) \indic{x_j, a_j = x, a},  \label{eqn:online-every-visit} 
\end{align}
where  $\delta^{\pi_k}_t := r_t + \gamma \E_{\pi_k} Q_k(x_{t+1}, \cdot) - Q_k(x_t, a_t)$, $\alpha_k = \alpha_k(x_s, a_s)$. We consider the Retrace($\lambda$) algorithm where $c_i=\lambda\min\big(1,\frac{\pi(a_i|x_i)}{\mu(a_i|x_i)}\big)$.
Assume that $(\pi_k)$ are increasingly greedy w.r.t.~$(Q_k)$ and are each $\epsilon_k$-away from the greedy policies $(\pi_{Q_k})$, 
i.e.~$\max_x\|\pi_k(\cdot|x)-\pi_{Q_k}(\cdot|x)\|_1\leq \epsilon_k$, with  $\epsilon_k\to 0$. Assume that $P^{\pi_k}$ and $P^{\pi_k\wedge\mu_k}$ asymptotically commute: $\lim_{k} \| P^{\pi_k}P^{\pi_k\wedge\mu_k} - P^{\pi_k\wedge\mu_k} P^{\pi_k}\|=0$. Assume further that (1) all states and actions are visited infinitely often: $\sum_{t\geq 0}\Pr\{ x_t, a_t = x, a\} \ge D > 0$, (2) the sample trajectories are finite in terms of the second moment of their lengths $T_k$: $\E_{\mu_k} T_k^2 < \infty$, 
  (3) the
stepsizes obey the usual Robbins-Munro conditions.
Then $Q_k \to Q^*$ a.s.
\end{theorem}
The proof extends similar convergence proofs of TD($\lambda$) by \citet{bertsekas1996neurodynamic} and of optimistic policy iteration by \cite{Tsitsiklis:2003}, and is provided in the appendix. Notice that compared to Theorem~\ref{thm:main-result} we do not assume that $\T^{\pi_0}Q_0-Q_0\geq0$ here. However, we make the additional (rather technical) assumption that $P^{\pi_k}$ and $P^{\pi_k\wedge\mu_k}$ commute at the limit. This is satisfied for example when the probability assigned by the behavior policy $\mu_k(\cdot|x)$ to the greedy action $\pi_{Q_k}(x)$ is independent of $x$. Examples include $\epsilon$-greedy policies, or more generally mixtures between the greedy policy $\pi_{Q_k}$ and an arbitrary distribution $\mu$ (see Lemma~5 in the appendix for the proof):
\beq \label{eq:example.mu}
\mu_k(a|x) = \varepsilon\frac{\mu(a|x)}{1-\mu(\pi_{Q_k}(x)|x)}\1\{a\neq \pi_{Q_k}(x)\} + (1-\varepsilon) \1\{a=\pi_{Q_k}(x)\}.
\eeq
Notice that the mixture coefficient $\varepsilon$ needs not go to $0$.

\section{Discussion of the results}

\subsection{Choice of the trace coefficients $c_s$} 

Theorems~\ref{thm:retrace.contraction} and~\ref{thm:main-result} ensure convergence to $Q^{\pi}$ and $Q^*$ for any trace coefficient $c_s\in [0,\frac{\pi(a_s|x_s)}{\mu(a_s|x_s)}]$. 
However, to make the best choice of $c_s$, we need to consider the {\em speed} of convergence, which depends on both (1) the variance of the online estimate, which indicates how many online updates are required in a single iteration of $\R$, and (2) the contraction coefficient of $\R$. 

\highlight{Variance:} The variance of the estimate strongly depends on the variance of the product trace $(c_1 \dots c_t)$, which 
is not an easy quantity to control in general, as the $(c_s)$ are usually not independent. However, 
assuming independence and stationarity of $(c_s)$, we have that $\Var\big(\sum_t\gamma^t c_1\dots c_t\big)$ is at least  $\sum_t \gamma^{2t} \Var(c)^t$, which is finite only if $\Var(c)<1/\gamma^2$. Thus, an important requirement for a numerically stable algorithm is for $\Var(c)$ to be as small as possible, and certainly no more than $1/\gamma^2$. This rules out importance sampling (for which $c=\frac{\pi(a|x)}{\mu(a|x)}$, and $\Var(c|x) = \sum_a \mu(a|x) \big(\frac{\pi(a|x)}{\mu(a|x)} - 1\big)^2$, which may be larger than $1/\gamma^2$ for some $\pi$ and $\mu$), and is the reason we choose $c\leq 1$.

\highlight{Contraction speed:} The contraction coefficient $\eta \in [0, \gamma]$ of $\R$ (see Remark~\ref{rem:eta}) depends on how much the traces have been cut, and should be as small as possible (since it takes $\log (1 / \epsilon) / \log (1 / \eta)$ iterations of $\R$ to obtain an $\epsilon$-approximation).
It is smallest when the traces are not cut at all (i.e.~if $c_s=1$ for all $s$, $\R$ is the policy evaluation operator which produces $Q^{\pi}$ in a single iteration). Indeed, when the traces are cut, we do not benefit from learning from full returns (in the extreme, $c_1=0$ and $\R$ reduces to the (one step) Bellman operator with $\eta=\gamma$).

A reasonable trade-off between low variance (when $c_s$ are small) and high contraction speed (when $c_s$ are large) is given by Retrace($\lambda$),
for which we provide the convergence of the online algorithm. 

If we relax the assumption that the trace is Markovian (in which case only the result for policy evaluation has been proven so far) we could trade off a low trace at some time for a possibly larger-than-$1$ trace at another time, as long as their product is less than $1$. A possible choice could be 
\beq\label{eq:non.markov.c}
c_s = \lambda \min\Big(\frac{1}{c_1\dots c_{s-1}},\frac{\pi(a_s|x_s)}{\mu(a_s|x_s)}\Big).
\eeq

\subsection{Other topics of discussion}

\paragraph{No GLIE assumption.}
The crucial point of Theorem~\ref{thm:main-result} is that convergence to $Q^*$ occurs for {\em arbitrary} behaviour policies. Thus the online result in Theorem~\ref{thm:online} does not require the behaviour policies to become greedy in the limit with infinite exploration \citep[i.e.~GLIE assumption,][]{singh2000convergence}. We believe Theorem~\ref{thm:online} provides the first convergence result to $Q^*$ for a $\lambda$-return (with $\lambda>0$) algorithm that does not require this (hard to satisfy) assumption.
\vspace{-0.3cm}

\paragraph{Proof of Watkins' Q($\lambda$).}
As a corollary of Theorem~\ref{thm:online} when selecting our target policies $\pi_k$ to be greedy w.r.t.~$Q_k$ (i.e.~$\epsilon_k=0$), we deduce that Watkins' Q($\lambda$) 
\citep[
e.g.,][]{Watkins:1989,sutton-barto98} converges a.s.~to $Q^*$ (under the assumption that $\mu_k$ commutes asymptotically with the greedy policies, which is satisfied for e.g.~$\mu_k$ defined by \eqref{eq:example.mu}). We believe this is the first such proof.
\vspace{-0.3cm}

\paragraph{Increasingly greedy policies}
The assumption that the sequence of target policies $(\pi_k)$ is increasingly greedy w.r.t.~the sequence of $(Q_k)$ is more general that just considering greedy policies w.r.t.~$(Q_k)$ (which is Watkins's Q($\lambda$)), and leads to more efficient algorithms. Indeed, using non-greedy target policies $\pi_k$ may speed up convergence as the traces are not cut as frequently. Of course, in order to converge to $Q^*$, we eventually need the target policies (and not the behaviour policies, as mentioned above) to become greedy in the limit (i.e.~$\epsilon_k\to 0$ as defined in Theorem~\ref{thm:main-result}).
\vspace{-0.3cm}

\paragraph{Comparison to $Q^\pi(\lambda)$.} 
Unlike Retrace($\lambda$), $Q^{\pi}(\lambda)$ does not need to know the behaviour policy $\mu$. However, it fails to converge when $\mu$ is far from $\pi$.
Retrace($\lambda$) uses its knowledge of $\mu$ (for the chosen actions) to cut the traces and safely handle arbitrary policies $\pi$ and $\mu$.

\vspace{-0.3cm}

\paragraph{Comparison to TB($\lambda$).} 
Similarly to $Q^{\pi}(\lambda)$, TB($\lambda$) does not need the knowledge of the behaviour policy $\mu$. But as a consequence, TB($\lambda$) is not able to benefit from possible near on-policy situations, cutting traces unnecessarily when $\pi$ and $\mu$ are close.
\vspace{-0.3cm}

\paragraph{Estimating the behavior policy.} In the case $\mu$ is unknown, it is reasonable to build an estimate $\hat \mu$ from observed samples and use $\hat\mu$ instead of $\mu$ in the definition of the trace coefficients $c_s$. This may actually even lead to a better estimate, as analyzed by \citet{LiAistats2015}.
\vspace{-0.3cm}

\paragraph{Continuous action space.} Let us mention that Theorems~\ref{thm:retrace.contraction} and~\ref{thm:main-result} extend to the case of (measurable) continuous or infinite action spaces. The trace coefficients will make use of the densities $\min(1,d\pi/d\mu)$ instead of the probabilities $\min(1,\pi/\mu)$. This is not possible with TB($\lambda$).
\vspace{-0.3cm}

\paragraph{Open questions include:} (1) Removing the technical assumption that $P^{\pi_k}$ and $P^{\pi_k\wedge\mu_k}$ asymptotically commute, (2) Relaxing the Markov assumption in the control case in order to allow trace coefficients $c_s$ of the form \eqref{eq:non.markov.c}.

\section{Experimental Results}

\begin{figure*}
\begin{center}
\includegraphics[width=2.2in]{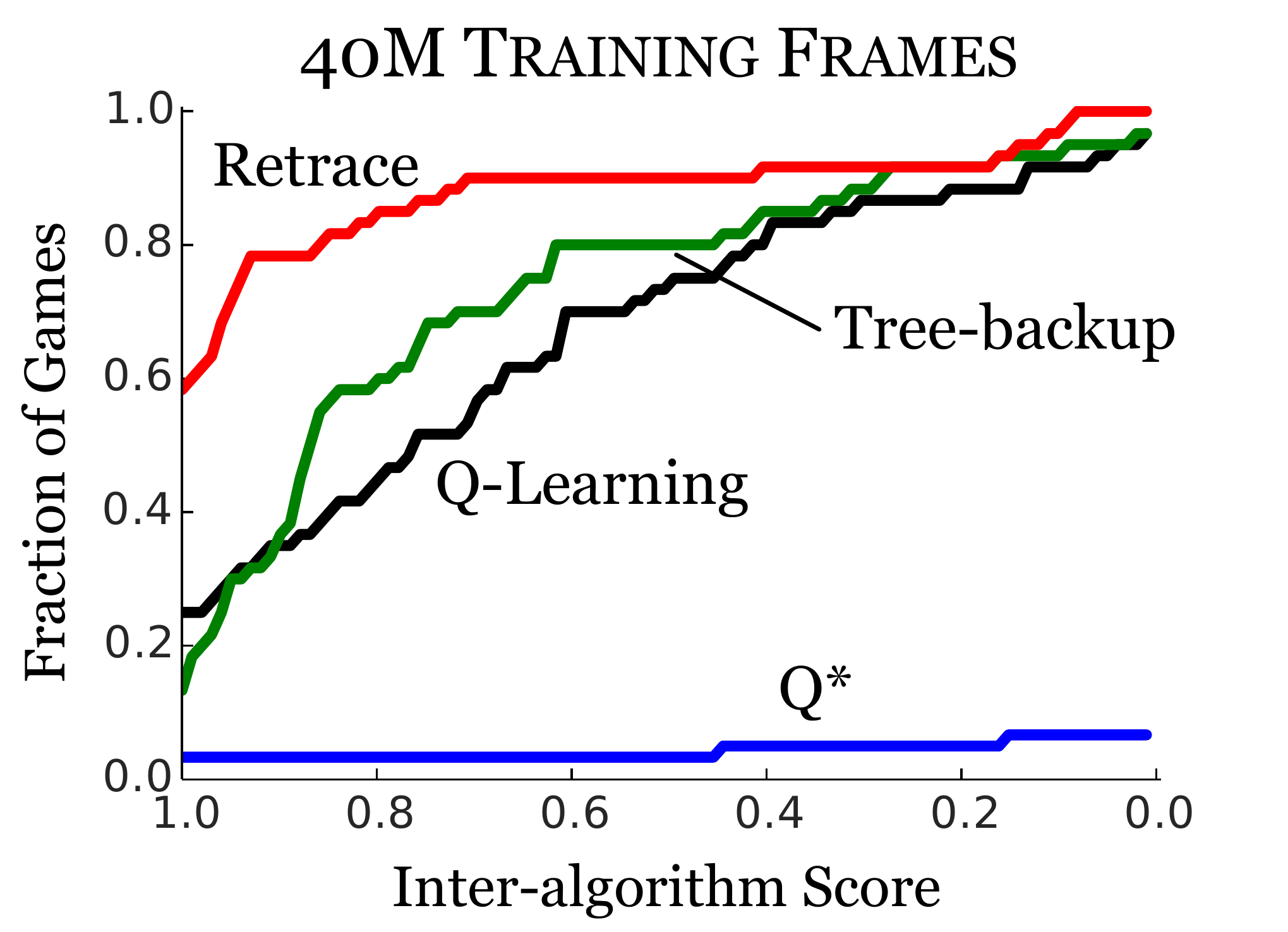}
\includegraphics[width=2.2in]{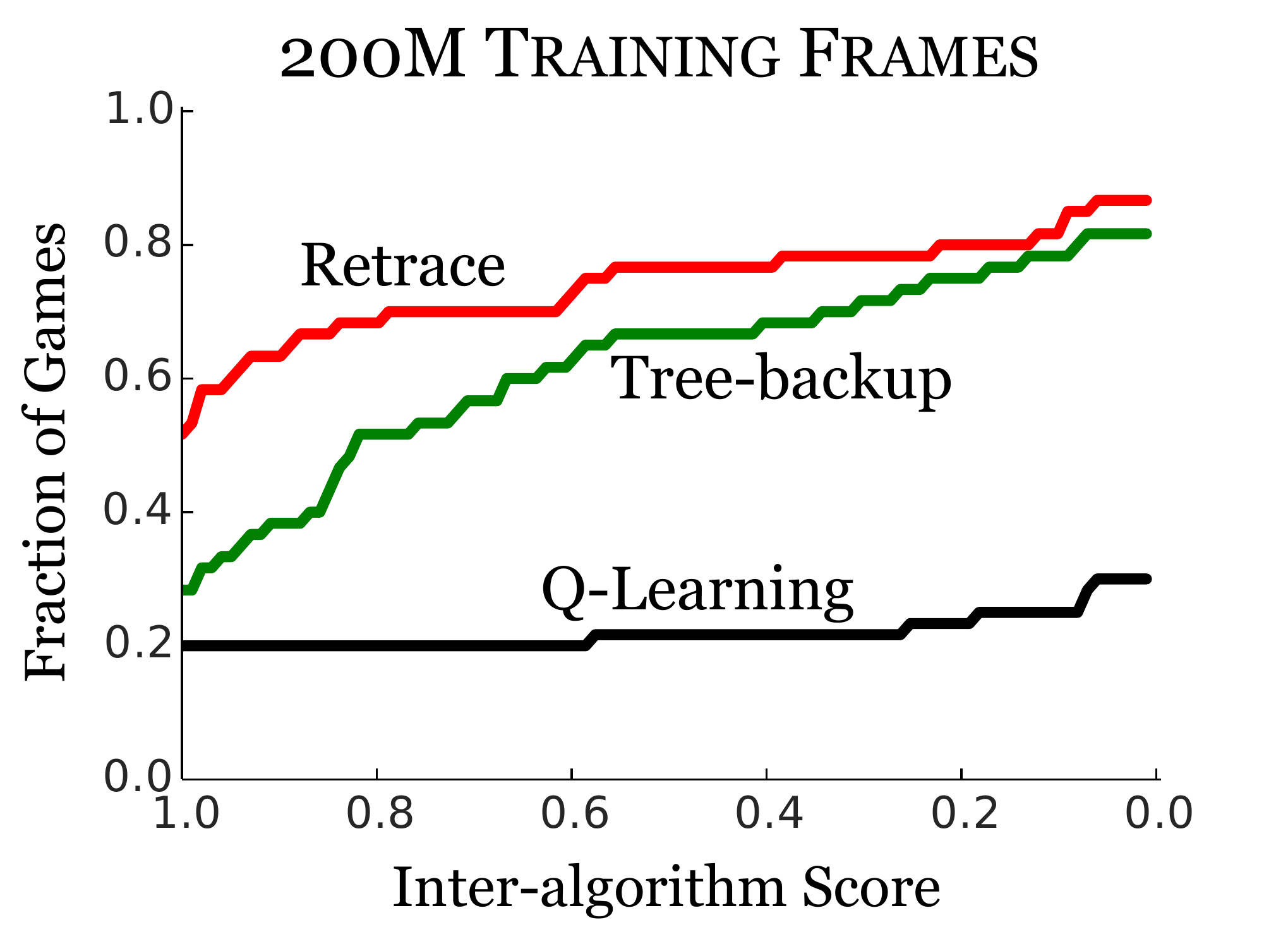}
\end{center}
\vspace{-1.0em}
\caption{Inter-algorithm score distribution for $\lambda$-return ($\lambda = 1$) variants and Q-Learning ($\lambda = 0$).\label{fig:score_distribution_200M}}
\end{figure*}

To validate our theoretical results, we employ Retrace$(\lambda)$ in an
experience replay \citep{lin93scaling} setting, where sample transitions are stored within
a large but bounded \emph{replay memory} and subsequently replayed as if they were new experience. 
Naturally, older data in the memory is usually drawn from a policy which differs from the current
policy, offering an excellent point of comparison for the algorithms presented in Section
\ref{sec:off_policy_algorithms}.

Our agent adapts the DQN architecture of \citet{mnih15human} to replay short
sequences from the memory (details in the appendix) instead of single
transitions. The Q-function target update for a sample sequence
$x_t, a_t, r_t, \cdots, x_{t+k}$ is
\begin{equation*}
\Delta Q(x_t, a_t) = \sum_{s=t}^{t+k-1} \gamma^{s-t} \Big( \prod_{i=t+1}^s c_i \Big) \big [ r(x_s, a_s) + \gamma \E_\pi Q(x_{s+1}, \cdot) - Q(x_s, a_s) \big ] .
\end{equation*}
We compare our algorithms' performance on 60 different Atari 2600 games in the Arcade Learning Environment \citep{bellemare13arcade} using \citeauthor{bellemare13arcade}'s inter-algorithm score distribution. Inter-algorithm scores are normalized so that 0 and 1 respectively correspond to the worst and best score for a particular game, within the set of algorithms under comparison. If $g \in \{ 1, \dots, 60 \}$ is a game and $z_{g,a}$ the inter-algorithm score on $g$ for algorithm $a$, then the score distribution function is $f(x) := | \{ g : z_{g,a} \ge x \} | / 60$.
Roughly, a strictly higher curve corresponds to a better algorithm. 

Across values of $\lambda$, $\lambda = 1$ performs best, save for $Q^*(\lambda)$ where $\lambda = 0.5$ obtains slightly superior performance. However, is highly sensitive to the choice of $\lambda$ (see Figure \ref{fig:score_distribution_200M}, left, and Table 2 in the appendix). Both Retrace($\lambda$) and TB$(\lambda)$ achieve dramatically higher performance than Q-Learning early on and maintain their advantage throughout.
Compared to TB($\lambda$), Retrace($\lambda$) offers a narrower but still marked advantage, being the best performer on 30 games; TB($\lambda$) claims 15 of the remainder.
Per-game details are given in the appendix.

\paragraph{Conclusion.} Retrace($\lambda$) can be seen as an algorithm that automatically adjusts -- efficiently and safely -- the length of the return to the degree of "off-policyness" of any available data.

\paragraph{Acknowledgments.} The authors thank Daan Wierstra, Nicolas Heess, Hado van Hasselt, Ziyu Wang, David Silver, Audrunas Gr\={u}slys, Georg Ostrovski, Hubert Soyer, and others at Google DeepMind for their very useful feedback on this work.

\bibliographystyle{apalike}
\bibliography{retrace}

\noappendix{
\appendix

\section{Proof of Lemma \ref{lem:retrace_delta_lemma}}

\begin{proof}[Proof (Lemma \ref{lem:retrace_delta_lemma})] 
Let $\Delta Q := Q-Q^{\pi}$. We begin by rewriting \eqnref{general_operator}:
\begin{align*}
\R Q(x,a) &=\sum_{t \ge 0} \gamma^t \E_\mu \left [ \Big ( \prod_{s=1}^t c_s \Big ) \Big ( r_t + \gamma \left [ \E_\pi Q(x_{t+1}, \cdot) - c_{t+1} Q(x_{t+1}, a_{t+1}) \Big ) \right ] \right ].
\end{align*}
Since $Q^\pi$ is the fixed point of $\R$, we have
\begin{align*}
Q^{\pi}(x,a) = \R Q^{\pi}(x,a) &=\sum_{t \ge 0} \gamma^t \E_\mu \left [ \Big ( \prod_{s=1}^t c_s \Big ) \Big ( r_t + \gamma \left [ \E_\pi Q^{\pi}(x_{t+1}, \cdot) - c_{t+1} Q^{\pi}(x_{t+1}, a_{t+1}) \Big ) \right ] \right ],
\end{align*}
from which we deduce that
\begin{align*}
\R Q(x,a)- Q^{\pi}(x,a)& =\sum_{t\geq 0} \gamma^{t} \E_{\mu}  \Big[ \Big( \prod_{s=1}^{t} c_s\Big) \Big( \gamma \big[ \E_\pi \Delta Q(x_{t+1}, \cdot) - c_{t+1} \Delta Q(x_{t+1},a_{t+1}) \big] \Big) \Big] \\
&= \sum_{t\geq 1} \gamma^{t} \E_{\mu} \Big[ \Big( \prod_{s=1}^{t-1} c_s\Big) \Big( \big[ \E_\pi \Delta Q(x_{t}, \cdot) - c_{t} \Delta Q(x_{t},a_{t}) \big] \Big)\Big].
\end{align*}
\end{proof}

\section{Increasingly greedy policies}
Recall the definition of an increasingly greedy sequence of policies.
\begin{predefinition}
 We say that a sequence of policies $(\pi_k)$ is increasingly greedy w.r.t.~a sequence of functions $(Q_k)$ if the following property holds for all $k$:
$$P^{\pi_{k+1}}Q_{k+1} \geq P^{\pi_{k}}Q_{k+1}.$$
\end{predefinition}

It is obvious to see that this property holds if all policies $\pi_k$ are greedy w.r.t.~$Q_k$. Indeed in such case, $\T^{\pi_{k+1}}Q_{k+1} = \T  Q_{k+1} \geq \T^{\pi} Q_{k+1}$ for any $\pi$.

We now prove that this property holds for $\epsilon_k$-greedy policies (with non-increasing $(\epsilon_k)$) as well as soft-max policies (with non-decreasing $(\beta_k)$), as stated in the two lemmas below. 

Of course not all policies satisfy this property (a counter-example being $\pi_k(a|x):=\arg\min_{a'}Q_k(x,a')$).

\begin{lemma}
Let $(\epsilon_k)$ be a non-increasing sequence. Then the sequence of policies $(\pi_k)$ which are $\epsilon_k$-greedy w.r.t.~the sequence of functions $(Q_k)$ is increasingly greedy w.r.t.~that sequence.
\end{lemma}
\begin{proof}
From the definition of an $\epsilon$-greedy policy we have:
\beqan
P^{\pi_{k+1}}Q_{k+1}(x,a) &=& \sum_y p(y|x,a) \big[ (1-\epsilon_{k+1})\max_b Q_{k+1}(y,b)+\epsilon_{k+1} \frac{1}{A}\sum_b Q_{k+1}(y,b) \big]\\
&\geq& \sum_y p(y|x,a) \big[ (1-\epsilon_{k})\max_b Q_{k+1}(y,b)+\epsilon_k \frac{1}{A}\sum_b Q_{k+1}(y,b) \big]\\
&\geq & \sum_y p(y|x,a) \big[ (1-\epsilon_{k}) Q_{k+1}(y,\arg\max_b Q_k(y,b))+\epsilon_k \frac{1}{A}\sum_b Q_{k+1}(y,b)\big] \\
& = & P^{\pi_{k}}Q_{k+1},
\eeqan
where we used the fact that $\epsilon_{k+1}\leq \epsilon_k$.
\end{proof}

\begin{lemma}
Let $(\beta_k)$ be a non-decreasing sequence of soft-max parameters. Then the sequence of policies $(\pi_k)$ which are soft-max (with parameter $\beta_k$) w.r.t.~the sequence of functions $(Q_k)$ is increasingly greedy w.r.t.~that sequence.
\end{lemma}
\begin{proof}
For any $Q$ and $y$, define $\pi_{\beta}(b) = \frac{e^{\beta Q(y,b)}}{\sum_{b'} e^{\beta Q(y,b')} }$ and 
$f(\beta) = \sum_b \pi_{\beta}(b) Q(y,b).$
Then we have
\beqan
f'(\beta) &=& \sum_b \big[\pi_\beta(b)Q(y,b) - \pi_\beta(b)\sum_{b'}\pi_\beta(b')Q(y, b')\big]Q(y,b) \\
&=& \sum_b \pi_\beta(b)Q(y,b)^2 - \big(\sum_{b}\pi_\beta(b)Q(y,b)\big)^2\\
&=& \Var_{b\sim \pi_\beta} \big[ Q(y,b)\big]\geq 0.
\eeqan

Thus $\beta\mapsto f(\beta)$ is a non-decreasing function, and since $\beta_{k+1}\geq \beta_k$, we have
\begin{align*}
P^{\pi_{k+1}}Q_{k+1}(x,a) &= \sum_y p(y|x,a) \sum_b \frac{e^{\beta_{k+1} Q_{k+1}(y,b)}}{\sum_{b'} e^{\beta_{k+1} Q_{k+1}(y,b')} } Q_{k+1}(y,b) \\
&\geq \sum_y p(y|x,a) \sum_b \frac{e^{\beta_{k} Q_{k+1}(y,b)}}{\sum_{b'} e^{\beta_{k} Q_{k+1}(y,b')} } Q_{k+1}(y,b) \\
&= P^{\pi_{k}}Q_{k+1}(x,a). \qedhere
\end{align*}
\end{proof}

\section{Proof of Theorem~\ref{thm:main-result}}

As mentioned in the main text, since $c_s$ is Markovian, we can define the (sub)-probability transition operator 
$$(P^{c\mu} Q)(x,a) := \sum_{x'}\sum_{a'} p(x'|x,a)\mu(a'|x') c(a',x') Q(x',a').$$
The Retrace$(\lambda)$ operator then writes
\begin{equation*}
\R_k Q = Q + \sum_{t \ge 0} \gamma^t (\Pcmu)^t (\Tpik Q - Q) = Q + (I - \gamma \Pcmu)^{-1} (\Tpik Q - Q).
\end{equation*}

\begin{proof}
We now lower- and upper-bound the term $Q_{k+1} - Q^*$.

\paragraph{Upper bound on $Q_{k+1} - Q^*$.} Since $Q_{k+1} = \R_k Q_k$, we have
\beqa
Q_{k+1} - Q^* &=& Q_k - Q^* + (I - \gamma \Pcmu)^{-1} \big [ \Tpik Q_k - Q_k \big ] \notag \\
&=& (I-\gamma P^{c\mu_k})^{-1} \big[ \T^{\pi_k} Q_k - Q_k + (I-\gamma P^{c\mu_k})(Q_k-Q^*) ] \notag \\
&=& (I-\gamma P^{c\mu_k})^{-1} \big[ \Tpik Q_k - Q^* - \gamma P^{c\mu_k}(Q_k-Q^*) ] \notag \\
&=& (I-\gamma P^{c\mu_k})^{-1} \big[ \Tpik Q_k - \T Q^* - \gamma P^{c\mu_k}(Q_k-Q^*) ] \notag \\
&\le& (I-\gamma P^{c\mu_k})^{-1} \big[ \gamma P^{\pi_k} (Q_k - Q^*) - \gamma P^{c\mu_k}(Q_k-Q^*) ] \notag\\
&=& \gamma (I-\gamma P^{c\mu_k})^{-1} \big[ P^{\pi_k} -  P^{c\mu_k} \big] (Q_k-Q^*), \notag\\
&=& A_k (Q_k-Q^*), \label{eq:A_k.bound}
\eeqa
where $A_k := \gamma (I-\gamma P^{c\mu_k})^{-1} \big[ P^{\pi_k} -  P^{c\mu_k} \big]$.  

Now let us prove that $A_k$ has non-negative elements, whose sum over each row is at most $\gamma$. Let $e$ be the vector with 1-components. By rewriting $A_k$ as $\gamma \sum_{t\geq 0} \gamma^t (P^{c\mu_{k}})^t (P^{\pi_{k}} - P^{c\mu_{k}})$ and noticing that
\beq
(P^{\pi_{k}} - P^{c\mu_{k}})e(x,a) = \sum_{x'}\sum_{a'} p(x'|x,a) [\pi_{k}(a'|x') - c(a',x')\mu_{k}(a'|x')] \geq 0,\label{eq:A.positive}
\eeq
it is clear that all elements of $A_k$ are non-negative.  We have
\beqa
A_k e &=&\gamma \sum_{t\geq 0} \gamma^t (P^{c\mu_{k}})^t \big[  P^{\pi_{k}} - P^{c\mu_{k}} \big]e\notag\\
& = & \gamma \sum_{t\geq 0} \gamma^t (P^{c\mu_{k}})^t e - \sum_{t\geq 0} \gamma^{t+1} (P^{c\mu_{k}})^{t+1} e\notag\\
& = & e - (1-\gamma) \sum_{t\geq 0} \gamma^t (P^{c\mu_{k}})^t e \notag\\
& \leq & \gamma e, \label{eq:sum.A_k.rows}
\eeqa
(since $\sum_{t\geq 0} \gamma^t (P^{c\mu_{k}})^t e \geq e$). Thus $A_k$ has non-negative elements, whose sum over each row, is at most $\gamma$. We deduce from \eqref{eq:A_k.bound} that $Q_{k+1} - Q^*$
is upper-bounded by a sub-convex combination of components of $Q_k - Q^*$; the sum of their 
coefficients is at most $\gamma$. Thus
\begin{equation}
Q_{k+1} - Q^* \le \gamma \| Q_k - Q^* \| e. \label{eq:upper-bound-Qk-Q*.apx}
\end{equation}

\paragraph{Lower bound on $Q_{k+1} - Q^*$.} We have
\beqa
Q_{k+1} &=& Q_k + (I- \gamma P^{c\mu_k})^{-1} (\T^{\pi_k}Q_k - Q_k) \notag \\
&=& Q_k + \sum_{i\geq 0}\gamma^i (P^{c\mu_k})^{i} (\T^{\pi_k}Q_k - Q_k) \notag \\
&=& \T^{\pi_k} Q_k + \sum_{i\geq 1} \gamma^i (P^{c\mu_k})^{i} (\T^{\pi_k}Q_k - Q_k)\notag  \\
&=& \T^{\pi_k} Q_k +  \gamma P^{c\mu_k} (I-\gamma P^{c\mu_k})^{-1} (\T^{\pi_k}Q_k - Q_k). \label{eq:lower.bound.1.apx}
\eeqa

Now, from the definition of $\epsilon_k$ we have $\T^{\pi_k}Q_k \geq \T Q_k-\epsilon_k\|Q_k\|\geq \T^{\pi^*}Q_k - \epsilon_k\|Q_k\|,$ thus
\begin{align}
Q_{k+1} - Q^* &= Q_{k+1} - \T^{\pi_k}  Q_k +\T^{\pi_k}  Q_k - \T^{\pi^*}  Q_k + \T^{\pi^*}  Q_k - \T^{\pi^*}  Q^{*} \notag \\
&\geq Q_{k+1} - \T^{\pi_k} Q_k  +\gamma P^{\pi^*}  (Q_k -  Q^{*}) - \epsilon_k\|Q_k\| e \notag
\end{align}
Using \eqref{eq:lower.bound.1.apx} we derive the lower bound:
\beq 
Q_{k+1} - Q^* \geq \gamma P^{c\mu_k} (I-\gamma P^{c\mu_k})^{-1} (\T^{\pi_k}Q_k - Q_k)  + \gamma P^{\pi^*}  (Q_k -  Q^{*}) - \epsilon_k\|Q_k\|.\label{eq:lb-Qk+1-Q*.apx}
\eeq

\paragraph{Lower bound on $\Tpik Q_k - Q_k$.} By hypothesis, $(\pi_k)$ is increasingly greedy w.r.t.~$(Q_k)$, thus
\beqa
\T^{\pi_{k+1}} Q_{k+1} - Q_{k+1} &\geq& \T^{\pi_k} Q_{k+1} - Q_{k+1} \notag \\ 
&=& \T^{\pi_k} \R_k Q_k - \R_k Q_k \notag \\
&=& r+(\gamma P^{\pi_k}-I) \R_k Q_k \notag \\
&=& r + (\gamma P^{\pi_k}-I) \big [ Q_k +(I-\gamma P^{c\mu_k})^{-1} (\T^{\pi_k} Q_k - Q_k) \big ] \notag \\
&=& \T^{\pi_k} Q_k - Q_k  + (\gamma P^{\pi_k}-I) (I-\gamma P^{c\mu_k})^{-1} (\T^{\pi_k} Q_k - Q_k) \notag \\
&=& \gamma \big[ P^{\pi_k} -   P^{c\mu_k} \big] (I-\gamma P^{c\mu_k})^{-1} (\T^{\pi_k} Q_k - Q_k) \notag \\
&=& B_k (\T^{\pi_k} Q_k - Q_k),\label{eq:lb.TQk+1-Qk+1.apx}
\eeqa
where $B_k:=\gamma [ P^{\pi_k} -   P^{c\mu_k} ] (I-\gamma P^{c\mu_k})^{-1}$.
Since $P^{\pi_k} - \Pcmu$ has non-negative elements (as proven in \eqref{eq:A.positive}) as well as $(I-\gamma P^{c\mu_k})^{-1}$, then $B_k$ has non-negative elements as well. Thus
\begin{equation*}
\T^{\pi_k} Q_k - Q_k \geq B_{k-1} B_{k-2} \dots B_{0} (\T^{\pi_0} Q_0 - Q_0) \geq 0,
\end{equation*}
since we assumed $T^{\pi_0} Q_0 - Q_0 \ge 0$. Thus \eqref{eq:lb-Qk+1-Q*.apx} implies that
\beqan
Q_{k+1} - Q^* &\geq& \gamma P^{\pi^*}  (Q_k -  Q^{*}) - \epsilon_k\|Q_k\|.
\eeqan
and combining the above with \eqnref{upper-bound-Qk-Q*.apx} we deduce 
$$\| Q_{k+1} - Q^*\| \leq \gamma\| Q_k -  Q^{*}\| + \epsilon_k\|Q_k\|.$$

Now assume that $\epsilon_k\rightarrow 0$. We first deduce that $Q_k$ is bounded. Indeed as soon as $\epsilon_k<(1-\gamma)/2$, we have 
$$\| Q_{k+1} \| \leq \| Q^*\| + \gamma\| Q_k -  Q^{*}\| + \frac{1-\gamma}{2}\|Q_k\|\leq 
(1+\gamma)\| Q^*\| + \frac{1+\gamma}{2} \|Q_k\|.$$
Thus $\limsup \| Q_{k} \| \leq \frac{1+\gamma}{1-(1+\gamma)/2} \| Q^*\|$.
Since $Q_k$ is bounded, we deduce that $\limsup Q_{k}=Q^*$.
\end{proof}

\section{Proof of Theorem~\ref{thm:online}}

 We first prove convergence of the general online algorithm.

\begin{theorem}\label{thm:online-general}
Consider the algorithm
\beq
Q_{k+1}(x, a) = (1-\alpha_k(x, a)) Q_k(x, a) + \alpha_k(x,a) (\R_k Q_k(x, a) + \omega_k(x, a) + \upsilon_k(x, a)),\label{eq:online.algo}
\eeq
and assume that (1) $\omega_k$ is a centered, ${\cal F}_k$-measurable noise term of
bounded variance, and (2) $\upsilon_k$ is bounded from above by $\theta_k(\|Q_k\| + 1)$, where $(\theta_k)$ is a random sequence that
converges to 0 a.s. Then, under the same assumptions as in Theorem~\ref{thm:online}, we have that $Q_k\rightarrow Q^*$ almost surely.
\end{theorem}

\begin{proof}
We write $\R$ for $\R_k$. Let us prove the result in three steps.

{\bf Upper bound on $\R Q_k - Q^*$.} The first part of the proof is similar to the proof of \eqref{eq:upper-bound-Qk-Q*.apx}, so we have
\beq
\R Q_{k} - Q^* \leq \gamma \| Q_k - Q^*\| e.\label{eq:upper-bound-RQk-Q*}
\eeq

{\bf Lower bound on $\R Q_{k} - Q^*$}. Again, similarly to \eqref{eq:lb-Qk+1-Q*.apx} we have
\beqa
\R Q_{k} - Q^* &\geq &  \gamma \lambda P^{\pi_k\wedge\mu_k} (I-\gamma \lambda P^{\pi_k\wedge\mu_k})^{-1} (\T^{\pi_k}Q_k - Q_k)  \notag \\
& & + \gamma P^{\pi^*}  (Q_k -  Q^{*}) - \epsilon_k \|Q_k\|.\label{eq:lb-RQk-Q*.bis}
\eeqa

{\bf Lower-bound on $\T^{\pi_k} Q_k - Q_k$.}
Since the sequence of policies $(\pi_k)$ is increasingly greedy w.r.t.~$(Q_k)$, we have
\beqa
\T^{\pi_{k+1}} Q_{k+1} - Q_{k+1} &\geq& \T^{\pi_k} Q_{k+1} - Q_{k+1}\notag \\
&=& (1-\alpha_k) \T^{\pi_k} Q_{k} +\alpha_k  \T^{\pi_k} (\R Q_k + \omega_k + \upsilon_k) - Q_{k+1} \notag\\
&=& (1-\alpha_k) (\T^{\pi_k} Q_{k} - Q_{k} ) +\alpha_k \big[ \T^{\pi_k} \R Q_k - \R Q_k +
\omega'_k + \upsilon'_k \big], \label{eq:u.k+1}
\eeqa
where $\omega'_k := (\gamma P^{\pi_k} - I) \omega_k$ and $\upsilon'_k := (\gamma P^{\pi_k} - I) \upsilon_k$. It is easy to see that both $\omega'_k$ and $\upsilon'_k$ continue to satisfy the assumptions on $\omega_k$, and $\upsilon_k$. Now, from the definition of the $\R$ operator, we have
\begin{align*}
\T^{\pi_k} \R Q_k - \R Q_k &= r+(\gamma P^{\pi_k}-I) \R Q_k \\
&= r + (\gamma P^{\pi_k}-I) \big [ Q_k +(I-\gamma \lambda P^{\pi_k\wedge\mu_k})^{-1} (\T^{\pi_k} Q_{k} - Q_{k} ) \big ] \\
&= \T^{\pi_k} Q_{k} - Q_{k}  + (\gamma P^{\pi_k}-I) (I-\gamma \lambda P^{\pi_k\wedge\mu_k})^{-1} (\T^{\pi_k} Q_{k} - Q_{k} ) \\
&= \gamma ( P^{\pi_k} - \lambda P^{\pi_k\wedge\mu_k} ) (I-\gamma \lambda P^{\pi_k\wedge\mu_k})^{-1} (\T^{\pi_k} Q_{k} - Q_{k} ).
\end{align*}
Using this equality into \eqref{eq:u.k+1} and writing $\xi_k :=\T^{\pi_k} Q_{k} - Q_{k}$, we have
\beq\label{eq:almost.xi.k+1}
\xi_{k+1} \geq (1-\alpha_k) \xi_k +\alpha_k \big[ B_k \xi_k + \omega'_k + \upsilon'_k \big],
\eeq
where $B_k := \gamma (P^{\pi_k}- \lambda P^{\pi_k\wedge\mu_k})(I-\gamma \lambda P^{\pi_k\wedge\mu_k})^{-1}$. The matrix $B_k$ is non-negative but may not be a contraction mapping (the sum of its components per row may be larger than $1$). Thus we cannot directly apply Proposition 4.5 of \cite{bertsekas1996neurodynamic}. However, as we have seen in the proof of Theorem~\ref{thm:main-result}, the matrix $A_k := \gamma (I-\gamma \lambda P^{\pi_k\wedge\mu_k})^{-1}(P^{\pi_k}- \lambda P^{\pi_k\wedge\mu_k})$ is a $\gamma$-contraction mapping. So now we relate $B_k$ to $A_k$ using our assumption that $P^{\pi_k}$ and $P^{\pi_k\wedge\mu_k}$ commute asymptotically, i.e.~$\|P^{\pi_k} P^{\pi_k\wedge\mu_k}-P^{\pi_k\wedge\mu_k}P^{\pi_k}\| = \eta_k$ with $\eta_k\rightarrow 0$. For any (sub)-transition matrices $U$ and $V$, we have
\beqan
U(I-\lambda\gamma V)^{-1} &=&\sum_{t\geq 0}(\lambda\gamma)^t U V^t \\
&=&\sum_{t\geq 0}(\lambda\gamma)^t \Big[ \sum_{s=0}^{t-1} V^s (UV-V U ) V^{t-s-1} + V^{t} U \Big]\\
&=&(I-\lambda\gamma V)^{-1}U+ \sum_{t\geq 0}(\lambda\gamma)^t \sum_{s=0}^{t-1} V^s (UV-V U ) V^{t-s-1}.
\eeqan
Replacing $U$ by $P^{\pi_k}$ and $V$ by $P^{\pi_k\wedge\mu_k}$, we deduce
$$\| B_k - A_k \|\leq \gamma \sum_{t\geq 0}t (\lambda\gamma)^t \eta_k = \gamma \frac{1}{(1 - \lambda \gamma)^2} \eta_k.$$

Thus, from \eqref{eq:almost.xi.k+1},
\beq\label{eq:xi.k+1}
\xi_{k+1} \geq (1-\alpha_k) \xi_k +\alpha_k \big[ A_k \xi_k + \omega'_k + \upsilon''_k \big],
\eeq
where $\upsilon''_k := \upsilon'_k + \gamma \sum_{t\geq 0}t (\lambda\gamma)^t \eta_k \|\xi_k\|$ continues to satisfy the assumptions on $\upsilon_k$ (since $\eta_k\rightarrow 0$).

Now, let us define another sequence $\xi'_k$ as follows: $\xi'_0=\xi_0$ and
$$\xi'_{k+1} =  (1-\alpha_k) \xi'_k +\alpha_k (A_k \xi'_k + \omega'_k + \upsilon''_k ).$$
We can now apply Proposition 4.5 of \cite{bertsekas1996neurodynamic} to the sequence $(\xi'_k)$. The matrices $A_k$ are non-negative, and the sum of their coefficients per row is bounded by $\gamma$, see \eqref{eq:sum.A_k.rows}, thus $A_k$ are $\gamma$-contraction mappings and have the same fixed point which is $0$. The noise
$\omega'_k$ is centered and $\F_k$-measurable and satisfies the bounded
variance assumption, and $\upsilon''_k$ is bounded above by
$(1+\gamma)\theta'_k(\|Q_k\|+1)$ for some $\theta'_k\rightarrow 0$. Thus $\lim_k \xi'_k = 0$ almost surely.

Now, it is straightforward to see that $\xi_k \geq \xi'_k$ for all $k\geq 0$. Indeed by induction, let us assume that $\xi_k\geq \xi'_k$. Then 
\beqan
\xi_{k+1}& \geq& (1-\alpha_k) \xi_k +\alpha_k (A_k \xi_k + \omega'_k + \upsilon''_k ) \\
&\geq& (1-\alpha_k) \xi'_k +\alpha_k (A_k \xi'_k +  \omega'_k + \upsilon''_k ) \\
&=& \xi'_{k+1},
\eeqan
since all elements of the matrix $A_k$ are non-negative. Thus we deduce that 
\beq\label{eq:liminf}
\liminf_{k\rightarrow\infty} \xi_k \geq \lim_{k\rightarrow\infty} \xi'_k = 0
\eeq

{\bf Conclusion.}
Using \eqref{eq:liminf} in \eqref{eq:lb-RQk-Q*.bis} we deduce the lower bound:
\begin{align}
  \liminf_{k\rightarrow\infty} \R Q_k - Q^* \geq \liminf_{k\rightarrow\infty} \gamma P^{\pi^*}  (Q_k -  Q^{*}),
\end{align}
almost surely. 
Now combining with the upper bound \eqref{eq:upper-bound-RQk-Q*} we deduce that 
$$ \| \R Q_k  - Q^*  \| \leq \gamma \| Q_k-Q^* \| + O(\epsilon_k \| Q_k \|) + O(\xi_k).$$
The last two terms can be incorporated to the $\upsilon_k(x,a)$ and $\omega_k(x,a)$ terms, respectively; we thus again apply Proposition 4.5 of \cite{bertsekas1996neurodynamic} to the sequence $(Q_k)$ defined by \eqref{eq:online.algo} and deduce that $Q_k\rightarrow Q^*$ almost surely.
\end{proof}

It remains to rewrite the update \eqref{eqn:online-every-visit} in the form of \eqref{eq:online.algo}, in order to apply Theorem~\ref{thm:online-general}. 

Let $z^k_{s, t}$ denote the accumulating trace~\citep{sutton-barto98}:
\begin{equation}
  \label{eq:1}
  z^k_{s, t} := \sum_{j = s}^t \gamma^{t - j} \Big (
  \prod_{i=j+1}^t c_i \Big ) \indic{(x_j, a_j) = (x_s, a_s)}. \notag
\end{equation}
Let us write $Q^o_{k+1}(x_s, a_s)$ to emphasize the online setting. Then \eqref{eqn:online-every-visit} can be written as 
\begin{align}
 Q^o_{k+1}(x_s, a_s) & \gets  Q^o_k(x_s, a_s) + \alpha_k(x_s, a_s) \sum_{t \ge s}
\delta^{\pi_k}_t z^k_{s, t}, \label{eqn:online-concise} \\
 \delta^{\pi_k}_t & := r_t + \gamma \E_{\pi_k} Q^o_k(x_{t+1}, \cdot) -
                   Q^o_k(x_t, a_t), \notag  
\end{align}
Using our assumptions on finite trajectories, and $c_i \leq 1$, we can show that:
\beq\label{lem:bounded-z}
\E\Big[ \sum_{t \ge s} z^k_{s, t} | \F_k \Big] < \E \left[ T_k^2 | \F_k \right] < \infty 
\eeq 
where $T_k$ denotes trajectory length.
Now, let $D_k := D_k(x_s, a_s) := \sum_{t \geq s} \P\{(x_t, a_t) = (x_s,
a_s)\}$. Then, using \eqref{lem:bounded-z}, we can show that the total update is bounded, and rewrite
  \beq
   \E_{\mu_k}\Big[\sum_{t \ge s} \delta^{\pi_k}_t z^k_{s, t}\Big] = D_k(x_s, a_s)\big( \R_k Q_k(x_s,a_s) - Q(x_s, a_s) \big). \notag 
  \eeq
Finally, using the above, and writing $\alpha_k = \alpha_k(x_s, a_s)$, \eqref{eqn:online-concise} can be rewritten in the desired form:
 \begin{align}
 Q^o_{k+1}(x_s, a_s) & \gets  (1 - \tilde{\alpha}_k) Q^o_k(x_s, a_s) + \tilde{\alpha}_k \big( \R_k Q^o_k(x_s,a_s) + \omega_k(x_s, a_s) + \upsilon_k(x_s, a_s) \big), \label{eqn:online} \\
 \omega_k(x_s, a_s) & := (D_k)^{-1}\left(\sum_{t \ge s} \delta^{\pi_k}_t z^k_{s, t} - \E_{\mu_k} \left [ \sum_{t \ge s} \delta^{\pi_k}_t z^k_{s, t} \right ] \right), \notag \\
 \upsilon_k(x_s, a_s) &:= (\tilde\alpha_k)^{-1} \big(Q^o_{k+1}(x_s, a_s) -
Q_{k+1}(x_s, a_s)\big), \notag \\
 \tilde{\alpha}_k & := \alpha_k D_k. \notag
\end{align}

It can be shown that the variance of the noise term $\omega_k$ is bounded, using~\eqref{lem:bounded-z} and the fact that the reward function is
bounded. It follows from Assumptions 1-3 that the modified stepsize sequence $(\tilde{\alpha}_k)$ satisfies the conditions of Assumption 1. The second noise term $\upsilon_k(x_s, a_s)$ measures the difference between online iterates and the corresponding offline values, and can be shown to satisfy the required assumption analogously to
the argument in the proof of Prop.~5.2 in \cite{bertsekas1996neurodynamic}. The proof relies on the eligibility
coefficients~\eqref{lem:bounded-z} and rewards being bounded, the trajectories being finite, and
the conditions on the stepsizes being satisfied.

We can thus apply Theorem~\ref{thm:online-general} to \eqref{eqn:online}, and conclude that the iterates $Q^o_k \rightarrow Q^*$ as $k\rightarrow\infty$, w.p. 1.

\section{Asymptotic commutativity of $P^{\pi_k}$ and $P^{\pi_k\wedge\mu_k}$}
\begin{lemma}\label{lem:commute}
 Let $(\pi_k)$ and $(\mu_k)$ two sequences of policies. If there exists $\alpha$ such that for all $x,a$, 
 \beq \label{eq:cond.policies}
 \min(\pi_k(a|x),\mu_k(a|x)) = \alpha \pi_k(a|x) + o(1),
 \eeq
then the transition matrices $P^{\pi_k}$ and $P^{\pi_k\wedge\mu_k}$ asymptotically commute: $\|P^{\pi_k} P^{\pi_k\wedge\mu_k} - P^{\pi_k\wedge\mu_k}P^{\pi_k}\| = o(1)$.
\end{lemma}

\begin{proof}
For any $Q$, we have
 \begin{align*}
 (P^{\pi_k} P^{\pi_k\wedge\mu_k})Q(x,a)&=\sum_y p(y|x,a)\sum_b \pi_k(b|y) \sum_z p(z|y,b)\sum_c (\pi_k\wedge\mu_k)(c|z) Q(z,c) \\
 &= \alpha \sum_y p(y|x,a)\sum_b \pi_k(b|y) \sum_z p(z|y,b)\sum_c \pi_k(c|z) Q(z,c) + \|Q\| o(1) \\
 &= \sum_y p(y|x,a)\sum_b (\pi_k\wedge\mu_k)(b|y) \sum_z p(z|y,b)\sum_c \pi_k(c|z) Q(z,c) + \|Q\| o(1)  \\
 &= (P^{\pi_k\wedge\mu_k} P^{\pi_k})Q(x,a) + \|Q\| o(1) . \qedhere
\end{align*}

\end{proof}

\begin{lemma}\label{lem:commute.2}
 Let $(\pi_{Q_k})$ a sequence of (deterministic) greedy policies w.r.t.~a sequence $(Q_k)$. Let $(\pi_k)$ a sequence of policies that are $\epsilon_k$ away from $(\pi_{Q_k})$, in the sense that, for all $x$, 
 $$ \| \pi_k(\cdot|x) - \pi_{Q_k}(x) \|_1 := 1- \pi_k(\pi_{Q_k}(x)|x) + \sum_{a\neq \pi_{Q_k}(x)} \pi_k(a|x) \leq \epsilon_k.$$ 
 Let $(\mu_k)$ a sequence of policies defined by:
 \beq 
\mu_k(a|x) = \frac{\alpha\mu(a|x)}{1-\mu(\pi_{Q_k}(x)|x)}\1\{a\neq \pi_{Q_k}(x)\} + (1-\alpha) \1\{a=\pi_{Q_k}(x)\},
\eeq
for some arbitrary policy $\mu$ and $\alpha\in [0,1]$. Assume $\epsilon_k\rightarrow 0$. Then the transition matrices $P^{\pi_k}$ and $P^{\pi_k\wedge\mu_k}$ asymptotically commute. 
\end{lemma}

\begin{proof}
The intuition is that asymptotically $\pi_k$ gets very close to the deterministic policy $\pi_{Q_k}$. In that case, the minimum distribution $(\pi_k\wedge\mu_k)(\cdot|x)$ puts a mass close to $1-\alpha$ on the greedy action $\pi_{Q_k}(x)$, and no mass on other actions, thus $(\pi_k\wedge\mu_k)$ gets very close to $(1-\alpha) \pi_k$, and Lemma~\ref{lem:commute} applies (with multiplicative constant $1-\alpha$).

Indeed, from our assumption that $\pi_k$ is $\epsilon$-away from $\pi_{Q_k}$ we have:
$$\pi_k(\pi_{Q_k}(x)|x) \geq 1-\epsilon_k \mbox{, and } \pi_k(a\neq\pi_{Q_k}(x)|x) \leq \epsilon_k.$$
We deduce that 
\beqan
 (\pi_k\wedge\mu_k)(\pi_{Q_k}(x)|x) &=& \min(\pi_k(\pi_{Q_k}(x)|x),1- \alpha) \\
 &=& 1-\alpha + O(\epsilon_k) \\
 &=& (1-\alpha) \pi_k(\pi_{Q_k}(x)|x) + O(\epsilon_k),
\eeqan
and 
\beqan
 (\pi_k\wedge\mu_k)(a\neq \pi_{Q_k}(x)|x) &=& O(\epsilon_k) \\
 &=& (1-\alpha) \pi_k(a|x) + O(\epsilon_k).
\eeqan

Thus Lemma~\ref{lem:commute} applies (with a multiplicative constant $1-\alpha$) and  $P^{\pi_k}$ and $P^{\pi_k\wedge\mu_k}$ asymptotically commute.  
\end{proof}

\section{Experimental Methods}\label{sec:experimental_methods}

Although our experiments' learning problem closely matches the DQN setting used
by \citet{mnih15human} (i.e. single-thread off-policy learning with large
replay memory), we conducted our trials in the multi-threaded, CPU-based
framework of \citet{mnih16asynchronous}, obtaining ample result data from
affordable CPU resources. Key differences from the DQN are as follows. Sixteen
threads with private environment instances train simultaneously; each infers
with and finds gradients w.r.t.~a local copy of the network parameters;
gradients then update a ``master'' parameter set and local copies are
refreshed. Target network parameters are simply shared globally. Each thread
has private replay memory holding 62,500 transitions (1/16\textsuperscript{th}
of DQN's total replay capacity). The optimizer is unchanged from
\citep{mnih16asynchronous}: ``Shared RMSprop'' with step size annealing to 0
over $3\times10^8$ environment frames (summed over threads). Exploration
parameter ($\epsilon$) behaviour differs slightly: every 50,000 frames, threads
switch randomly (probability 0.3, 0.4, and 0.3 respectively) between three
schedules (anneal $\epsilon$ from 1 to 0.5, 0.1, or 0.01 over 250,000 frames),
starting new schedules at the intermediate positions where they left old
ones.\footnote{We evaluated a DQN-style single schedule for $\epsilon$, but our
multi-schedule method, similar to the one used by \citeauthor{mnih16asynchronous},
yielded improved performance in our multi-threaded setting.}

Our experiments comprise 60 Atari 2600 games in ALE \citep{bellemare13arcade},
with ``life'' loss treated as episode termination. The control, minibatched (64
transitions/minibatch) one-step Q-learning as in \citep{mnih15human}, shows
performance comparable to DQN in our multi-threaded setup. Retrace, TB, and
Q\textsuperscript{*} runs use minibatches of four 16-step sequences (again 64
transitions/minibatch) and the current exploration policy as the target policy
$\pi$. All trials clamp rewards into $[-1, 1]$. In the control, Q-function
targets are clamped into $[-1, 1]$ prior to gradient calculation; analogous
quantities in the multi-step algorithms are clamped into $[-1, 1]$, then scaled
(divided by) the sequence length. Coarse, then fine logarithmic parameter
sweeps on the games \emph{Asterix}, \emph{Breakout}, \emph{Enduro},
\emph{Freeway}, \emph{H.E.R.O}, \emph{Pong}, \emph{Q*bert}, and \emph{Seaquest}
yielded step sizes of 0.0000439 and 0.0000912, and RMSprop regularization
parameters of 0.001 and 0.0000368, for control and multi-step algorithms
respectively. Reported performance averages over four trials with different
random seeds for each experimental configuration.

\subsection{Algorithmic Performance in Function of $\lambda$}

We compared our algorithms for different values of $\lambda$, using the DQN score as a baseline. As before, for each $\lambda$ we compute the inter-algorithm scores on a per-game basis. We then averaged the inter-algorithm scores across games to produce Table \ref{table:scores_in_function_of_lambda} (see also Figure \ref{fig:scores_in_function_of_lambda} for a visual depiction). We first remark that Retrace always achieve a score higher than TB, demonstrating that it is efficient in the sense of Section \ref{sec:off_policy_algorithms}. Next, we note that $Q^*$ performs best for small values of $\lambda$, but begins to fail for values above $\lambda = 0.5$. In this sense, it is also not safe. This is particularly problematic as the safe threshold of $\lambda$ is likely to be problem-dependent. Finally, there is no setting of $\lambda$ for which Retrace performs particularly poorly; for high values of $\lambda$, it achieves close to the top score in most games.  For Retrace($\lambda$) it makes sense to use a values $\lambda=1$ (at least in deterministic environments) as the trace cutting effect required in off-policy learning is taken care of by the use of the $\min(1,\pi/\mu)$ coefficient. On the contrary, $Q^*(\lambda)$ only relies on a value $\lambda<1$ to take care of cutting traces for off-policy data.

\begin{table}[htb!]
\center{
\begin{tabular}{|r|r|r|r|r|}
\hline
$\lambda$ & DQN & TB & Retrace & $Q^*$ \\
\hline
\hline
0.0 &         0.5071 &  0.5512 &   0.4288 &   0.4487 \\
\hline
0.1 &         0.4752 &  0.2798 &   0.5046 &    0.651 \\
\hline
0.3 &         0.3634 &   0.268 &   0.5159 &   0.7734 \\
\hline
0.5 &         0.2409 &  0.4105 &   0.5098 &   0.8419 \\
\hline
0.7 &         0.3712 &  0.4453 &   0.6762 &   0.5551 \\
\hline
0.9 &         0.7256 &  0.7753 &   0.9034 &  0.02926 \\
\hline
1.0 &         0.6839 &  0.8158 &   0.8698 &  0.04317 \\
\hline
\end{tabular}
\vspace{0.5em}
}
\caption{Average inter-algorithm scores for each value of $\lambda$. The DQN scores are fixed across different $\lambda$, but the corresponding inter-algorithm scores varies depending on the worst and best performer within each $\lambda$.\label{table:scores_in_function_of_lambda}}
\end{table}

\begin{figure*}[htb]
\center{
\includegraphics{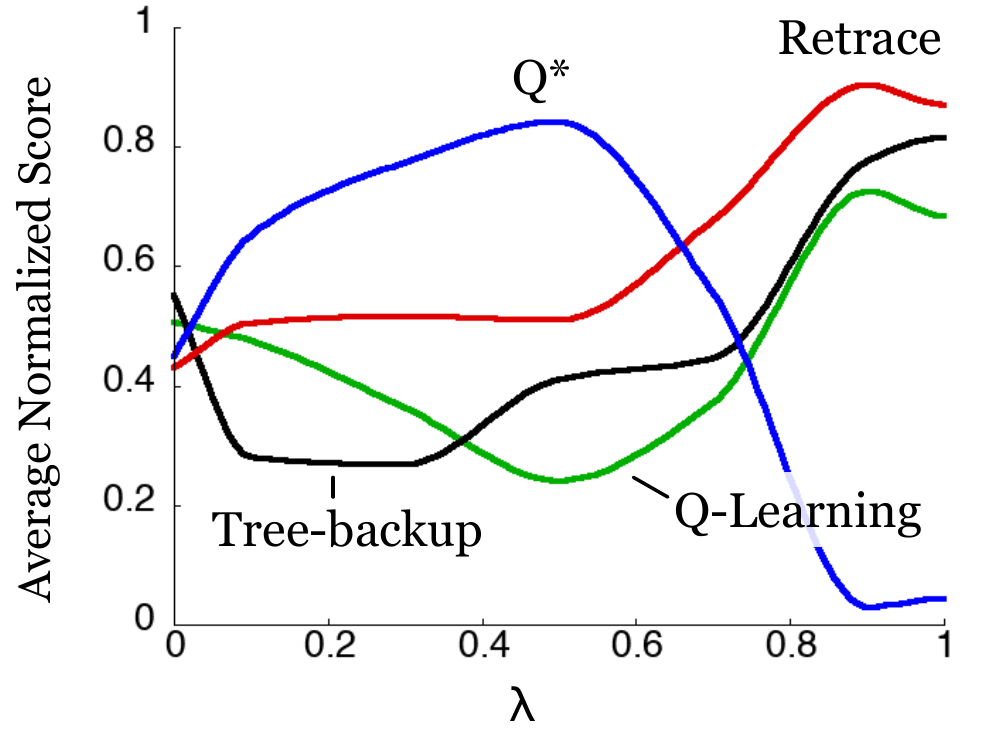}
}
\caption{Average inter-algorithm scores for each value of $\lambda$. The DQN scores are fixed across different $\lambda$, but the corresponding inter-algorithm scores varies depending on the worst and best performer within each $\lambda$. \textbf{Note that average scores are not directly comparable across different values of $\lambda$.}\label{fig:scores_in_function_of_lambda}}
\end{figure*}

\begin{table}
\center{
\small
\begin{tabular}{|r|r|r|r|r|}
\hline
& Tree-backup($\lambda$) & Retrace($\lambda$) & DQN & Q$^*$($\lambda$) \\
\hline
\textsc{Alien} &  2508.62  & \textbf{\textcolor{blue}{3109.21}}  &  2088.81  &  154.35   \\
\hline
\textsc{Amidar} &  1221.00  & \textbf{\textcolor{blue}{1247.84}}  &  772.30  &  16.04   \\
\hline
\textsc{Assault} &  7248.08  & \textbf{\textcolor{blue}{8214.76}}  &  1647.25  &  260.95   \\
\hline
\textsc{Asterix} & \textbf{\textcolor{blue}{29294.76}}  &  28116.39  &  10675.57  &  285.44   \\
\hline
\textsc{Asteroids} &  1499.82  & \textbf{\textcolor{blue}{1538.25}}  &  1403.19  &  308.70   \\
\hline
\textsc{Atlantis} & \textbf{\textcolor{blue}{2115949.75}}  &  2110401.90  &  1712671.88  &  3667.18   \\
\hline
\textsc{Bank Heist} & \textbf{\textcolor{blue}{808.31}}  &  797.36  &  549.35  &  1.70   \\
\hline
\textsc{Battle Zone} &  22197.96  & \textbf{\textcolor{blue}{23544.08}}  &  21700.01  &  3278.93   \\
\hline
\textsc{Beam Rider} &  15931.60  & \textbf{\textcolor{blue}{17281.24}}  &  8053.26  &  621.40   \\
\hline
\textsc{Berzerk} &  967.29  & \textbf{\textcolor{blue}{972.67}}  &  627.53  &  247.80   \\
\hline
\textsc{Bowling} &  40.96  & \textbf{\textcolor{blue}{47.92}}  &  37.82  &  15.16   \\
\hline
\textsc{Boxing} &  91.00  &  93.54  & \textbf{\textcolor{blue}{95.17}}  &  -29.25   \\
\hline
\textsc{Breakout} &  288.71  &  298.75  & \textbf{\textcolor{blue}{332.67}}  &  1.21   \\
\hline
\textsc{Carnival} & \textbf{\textcolor{blue}{4691.73}}  &  4633.77  &  4637.86  &  353.10   \\
\hline
\textsc{Centipede} &  1199.46  &  1715.95  &  1037.95  & \textbf{\textcolor{blue}{3783.60}}   \\
\hline
\textsc{Chopper Command} &  6193.28  & \textbf{\textcolor{blue}{6358.81}}  &  5007.32  &  534.83   \\
\hline
\textsc{Crazy Climber} & \textbf{\textcolor{blue}{115345.95}}  &  114991.29  &  111918.64  &  1136.21   \\
\hline
\textsc{Defender} &  32411.77  & \textbf{\textcolor{blue}{33146.83}}  &  13349.26  &  1838.76   \\
\hline
\textsc{Demon Attack} &  68148.22  & \textbf{\textcolor{blue}{79954.88}}  &  8585.17  &  310.45   \\
\hline
\textsc{Double Dunk} & \textbf{\textcolor{blue}{-1.32}}  &  -6.78  &  -5.74  &  -23.63   \\
\hline
\textsc{Elevator Action} &  1544.91  &  2396.05  & \textbf{\textcolor{blue}{14607.10}}  &  930.38   \\
\hline
\textsc{Enduro} &  1115.00  & \textbf{\textcolor{blue}{1216.47}}  &  938.36  &  12.54   \\
\hline
\textsc{Fishing Derby} &  22.22  & \textbf{\textcolor{blue}{27.69}}  &  15.14  &  -98.58   \\
\hline
\textsc{Freeway} & \textbf{\textcolor{blue}{32.13}}  &  32.13  &  31.07  &  9.86   \\
\hline
\textsc{Frostbite} &  960.30  &  935.42  & \textbf{\textcolor{blue}{1124.60}}  &  45.07   \\
\hline
\textsc{Gopher} &  13666.33  & \textbf{\textcolor{blue}{14110.94}}  &  11542.46  &  50.59   \\
\hline
\textsc{Gravitar} &  30.18  &  29.04  & \textbf{\textcolor{blue}{271.40}}  &  13.14   \\
\hline
\textsc{H.E.R.O.} & \textbf{\textcolor{blue}{25048.33}}  &  21989.46  &  17626.90  &  12.48   \\
\hline
\textsc{Ice Hockey} & \textbf{\textcolor{blue}{-3.84}}  &  -5.08  &  -4.36  &  -15.68   \\
\hline
\textsc{James Bond} &  560.88  &  641.51  & \textbf{\textcolor{blue}{705.55}}  &  21.71   \\
\hline
\textsc{Kangaroo} &  11755.01  & \textbf{\textcolor{blue}{11896.25}}  &  4101.92  &  178.23   \\
\hline
\textsc{Krull} & \textbf{\textcolor{blue}{9509.83}}  &  9485.39  &  7728.66  &  429.26   \\
\hline
\textsc{Kung-Fu Master} &  25338.05  & \textbf{\textcolor{blue}{26695.19}}  &  17751.73  &  39.99   \\
\hline
\textsc{Montezuma's Revenge} & \textbf{\textcolor{blue}{0.79}}  &  0.18  &  0.10  &  0.00   \\
\hline
\textsc{Ms. Pac-Man} &  2461.10  & \textbf{\textcolor{blue}{3208.03}}  &  2654.97  &  298.58   \\
\hline
\textsc{Name This Game} & \textbf{\textcolor{blue}{11358.81}}  &  11160.15  &  10098.85  &  1311.73   \\
\hline
\textsc{Phoenix} &  13834.27  & \textbf{\textcolor{blue}{15637.88}}  &  9249.38  &  107.41   \\
\hline
\textsc{Pitfall} & \textbf{\textcolor{blue}{-37.74}}  &  -43.85  &  -392.63  &  -121.99   \\
\hline
\textsc{Pooyan} &  5283.69  & \textbf{\textcolor{blue}{5661.92}}  &  3301.69  &  98.65   \\
\hline
\textsc{Pong} & \textbf{\textcolor{blue}{20.25}}  &  20.20  &  19.31  &  -20.99   \\
\hline
\textsc{Private Eye} &  73.44  & \textbf{\textcolor{blue}{87.36}}  &  44.73  &  -147.49   \\
\hline
\textsc{Q*Bert} &  13617.24  & \textbf{\textcolor{blue}{13700.25}}  &  12412.85  &  114.84   \\
\hline
\textsc{River Raid} &  14457.29  & \textbf{\textcolor{blue}{15365.61}}  &  10329.58  &  922.13   \\
\hline
\textsc{Road Runner} &  34396.52  &  32843.09  & \textbf{\textcolor{blue}{50523.75}}  &  418.62   \\
\hline
\textsc{Robotank} &  36.07  &  41.18  & \textbf{\textcolor{blue}{49.20}}  &  5.77   \\
\hline
\textsc{Seaquest} &  3557.09  &  2914.00  & \textbf{\textcolor{blue}{3869.30}}  &  175.29   \\
\hline
\textsc{Skiing} &  -25055.94  &  -25235.75  &  -25254.43  & \textbf{\textcolor{blue}{-24179.71}}   \\
\hline
\textsc{Solaris} &  1178.05  &  1135.51  & \textbf{\textcolor{blue}{1258.02}}  &  674.58   \\
\hline
\textsc{Space Invaders} & \textbf{\textcolor{blue}{6096.21}}  &  5623.34  &  2115.80  &  227.39   \\
\hline
\textsc{Star Gunner} &  66369.18  & \textbf{\textcolor{blue}{74016.10}}  &  42179.52  &  266.15   \\
\hline
\textsc{Surround} & \textbf{\textcolor{blue}{-5.48}}  &  -6.04  &  -8.17  &  -9.98   \\
\hline
\textsc{Tennis} &  -1.73  &  -0.30  & \textbf{\textcolor{blue}{13.67}}  &  -7.37   \\
\hline
\textsc{Time Pilot} &  8266.79  & \textbf{\textcolor{blue}{8719.19}}  &  8228.89  &  657.59   \\
\hline
\textsc{Tutankham} &  164.54  & \textbf{\textcolor{blue}{199.25}}  &  167.22  &  2.68   \\
\hline
\textsc{Up and Down} &  14976.51  & \textbf{\textcolor{blue}{18747.40}}  &  9404.95  &  530.59   \\
\hline
\textsc{Venture} &  10.75  &  22.84  & \textbf{\textcolor{blue}{30.93}}  &  0.09   \\
\hline
\textsc{Video Pinball} &  103486.09  & \textbf{\textcolor{blue}{228283.79}}  &  76691.75  &  6837.86   \\
\hline
\textsc{Wizard Of Wor} &  7402.99  & \textbf{\textcolor{blue}{8048.72}}  &  612.52  &  189.43   \\
\hline
\textsc{Yar's Revenge} &  14581.65  & \textbf{\textcolor{blue}{26860.57}}  &  15484.03  &  1913.19   \\
\hline
\textsc{Zaxxon} &  12529.22  & \textbf{\textcolor{blue}{15383.11}}  &  8422.49  &  0.40   \\
\hline
\hline
Times Best  & 16 & 30 & 12 & 2 \\
\hline
\end{tabular}
}
\caption{Final scores achieved by the different $\lambda$-return variants ($\lambda = 1$). Highlights indicate high scores.\label{fig:all_results}}
\label{tbl:performance_results}
\end{table}
}

\end{document}